\documentclass{elsarticle}

\usepackage{lineno}
\usepackage[colorlinks=false]{hyperref}
\usepackage{adjustbox}
\usepackage{graphicx}
\usepackage{amsmath,amsthm,amssymb}
\usepackage{algorithmic}
\usepackage{algorithm}
\usepackage{color,soul}
\usepackage{eurosym}
\usepackage{empheq}
\usepackage{csquotes}
\usepackage{tabularx}
\usepackage{float}
\usepackage{subfig}
\usepackage{cancel}
\usepackage[american]{babel}
\usepackage{booktabs}
\usepackage{stackengine}
\usepackage{enumerate}
\usepackage{lipsum}
\usepackage{setspace}
\usepackage{url} 
\usepackage{mathtools}
\usepackage{bm}
\usepackage{times}
\usepackage{mathabx}
\usepackage{enumerate}
\usepackage{tablefootnote}

\usepackage[firstpageonly=true]{draftwatermark}

\newtheorem{theorem}{Theorem}
\newtheorem{lemma}{Lemma}
\newtheorem{proposition}{Proposition}
\newtheorem{property}{Property}
\newtheorem{remark}{Remark}
\newtheorem{definition}{Definition}

\newcommand\scalemath[2]{\scalebox{#1}{\mbox{\ensuremath{\displaystyle #2}}}}
\newcommand{\identity}[0]{{\scriptsize I}d}
\newcommand{\REV}[1]{{#1}}

\hypersetup{
	colorlinks,
	citecolor=black,
	filecolor=black,
	linkcolor=black,
	urlcolor=black
}

\begin{document}
 \DraftwatermarkOptions{%
 angle=0,
 hpos=0.5\paperwidth,
 vpos=0.925\paperheight,
 fontsize=0.012\paperwidth,
 color={[gray]{0.2}},
 text={
   \parbox{0.99\textwidth}{\copyright \, 2022. This manuscript version is made available under the \href{https://creativecommons.org/licenses/by-nc-nd/4.0}{CC-BY-NC-ND 4.0 license}. This manuscript has been published at Elsevier European Journal of Control. Published article available at DOI \href{https://doi.org/10.1016/j.ejcon.2022.100632}{10.1016/j.ejcon.2022.100632}.}},
 }

\begin{frontmatter}
\journal{European Journal of Control}
\title{Recurrent Neural Network-based Internal Model Control design for stable nonlinear systems}
\author[Polimi]{Fabio Bonassi\corref{cor1}}
\author[Polimi]{Riccardo Scattolini}
\address[Polimi]{Dipartimento di Elettronica, Informazione e Bioingegneria, Politecnico di Milano, Via Ponzio 34/5, 20133 Milano, Italy - email: name.surname@polimi.it}
\cortext[cor1]{Corresponding author}

\begin{abstract}
    Owing to their superior modeling capabilities, gated Recurrent Neural Networks, such as Gated Recurrent Units (GRUs) and Long Short-Term Memory  networks (LSTMs), have become popular tools for learning dynamical systems.
    This paper aims to discuss how these networks can be adopted for the synthesis of Internal Model Control (IMC) architectures.
    To this end, first a gated recurrent network is used to learn a model of the unknown input-output stable plant. 
    Then, a controller gated recurrent network is trained to approximate the model inverse.
    \REV{The stability of these networks, ensured by means of a suitable training procedure, allows to guarantee the input-output closed-loop stability.}
    The proposed scheme is able to cope with the saturation of the control variables, and can be deployed on low-power embedded controllers, \REV{as it requires limited online computations}.
    The approach is then tested on the Quadruple Tank benchmark system \REV{and compared to alternative control laws, resulting in remarkable closed-loop performances.}
\end{abstract}

\end{frontmatter}

\section{Introduction}
% Neural networks for system identification and control
In recent years, Neural Networks (NNs) have gained popularity in the control systems community.
The key factors behind this success are the flexibility and modeling power of these tools, which make them suitable for system identification \cite{forgione2020model, rehmer2019using} and control tasks \cite{hunt1992neural, ali2015artificial}.

Through the control system designer lenses, the wide variety of neural networks' architectures can be classified in Feed-Forward Neural Networks (FFNNs), and Recurrent Neural Networks (RNNs) \cite{goodfellow2016deep}.
The former group is composed of static NNs, which do not have inherent memory and are mainly used to approximate static functions, for which task they enjoy universal approximation capabilities \cite{hornik1989multilayer}.
On the contrary, RNNs are stateful NNs, i.e., they retain memory of the past data, and hence they are particularly suitable for learning timeseries and dynamical systems \cite{bianchi2017recurrent, yu2004nonlinear, ogunmolu2016nonlinear}. 
In these tasks, RNNs have been shown to feature universal approximation capabilities \cite{schafer2006recurrent}.

% Gated recurrent neural networks
Among RNNs, the most popular architectures are the so-called gated recurrent networks, designed to avoid the vanishing/exploding gradient problems, which instead plague traditional RNN architectures \cite{hochreiter1998vanishing}. 
More specifically, Gated Recurrent Units (GRUs) \cite{chung2014empirical} and Long Short-Term Memory (LSTM) \cite{hochreiter1997long} are acknowledged to be the go-to architectures for system identification tasks \cite{forgione2020model}.

% Application of NNs for model-based control strategies
Owing to their modeling capabilities, NNs have been widely adopted in conjunction with model-based control strategies. 
% - Popularity for MPC
For example, the use of RNNs as system models for Model Predictive Control (MPC) design has been explored both in the industry \cite{hosen2011control, wong2018recurrent, lanzetti2019recurrent}, and in the academia \cite{wu2019machine, patan2014neural, terzi2021learning, bonassi2021nonlinear}.
\REV{hile MPC generally attains remarkable performances, constraint satisfaction and -- if suitably designed -. closed-loop stability guarantees \cite{terzi2021learning}, it requires to solve an online nonlinear optimization problem at each control step, which might be computationally prohibitive in some applications.}

An alternative control architecture in which NNs can be profitably employed is the Internal Model Control (IMC) scheme.
This architecture is particularly interesting because of its limited online computational cost, which makes it suitable to the deployment to embedded controllers \REV{with limited computational resources}.
In the IMC scheme, a reliable model of the (unknown) plant is assumed to be available, and the controller is constructed as the inverse (or, at least, an approximation of the inverse) of such model \cite{economou1986internal}.
The controller is supposed to generate the control action that steers the model output close to the reference trajectory.
% This controller is then used to generate the control action realising the target reference trajectory.
In order to avoid open-loop operations, the plant-model mismatch is fed back and summed to the reference signal \cite{economou1986internal, morari1989robust}, thus leveraging  output measurements to improve the robustness of the scheme.

The aim of this paper is to show how RNNs can be exploited to synthesize the ingredients of the IMC scheme.
In particular, we propose to use a GRU network to learn a model of the unknown plant, and then to learn a controller GRU network as an approximation of the model's inverse itself.
\REV{The proposed approach moves from traditional FFNN-based approaches,} such as that proposed in \cite{hunt1991neural}, where static FFNN have been used as system's model and model's inverse approximators, and in \cite{rivals2000nonlinear}, where FFNNs have been used in an auto-regressive configuration to control a stable SISO system. \REV{With respect to existing methods, our approach yields the following advantages.
\begin{enumerate}[i.]
	\item It is based on RNNs, which represent better candidates for learning dynamical systems (such as the model and its inverse), thanks to the retained long-term memory of the past trajectories. In contrast, for feed-forward auto-regressive architectures, this memory must be enforced by supplying the past input-output data-points as inputs of the network \cite{rivals2000nonlinear}, typically resulting in less accurate long-term learning compared to RNNs \cite{bonassi2020nnarx}. 
	\item The controller is implemented by a gated recurrent network, which is inherently strictly proper, as opposed to feed-forward architectures where one needs to deal with the issue of the controller's improperness in the controller design phase \cite{rivals2000nonlinear}, e.g. in presence of delays.
	\item Owing to the model and controller learning procedures, the proposed scheme can handle MIMO systems, while previous approaches have been formulated for SISO systems. In addition, our method accounts for input saturation constraints.
\end{enumerate}}
The clear advantages of using gated recurrent networks come, however, at the expense of losing the guarantees on the existence of the exact model's inverse.
Nonetheless, we show that a sufficiently accurate approximation of this inverse is enough to ensure closed-loop stability and satisfactory performances.

% Stability
Under the assumption of input-output stability\footnote{\REV{The term \emph{input-output stability} is here adopted in the sense of finite-gain $\mathcal{L}_p$ stability \cite{khalil2002nonlinear}, meaning that any bounded input $u$ leads to a bounded output $y$.}} of the plant to be controlled, we propose to enforce a similar stability property for both the learned system model and the  controller, namely the Incremental Input-to-State Stability ($\delta$ISS)\footnote{The $\delta$ISS property guarantees that the effects of different initial conditions asymptotically vanish, and that the smaller the maximum euclidean distance between two input sequences, the closer are the resulting state (and output) trajectories, as more formally described in Section \ref{sec:rnn:stability}.} \cite{bayer2013discrete}.
\REV{To this end, we show how to train provenly $\delta$ISS model and controller RNNs, by enforcing the fulfillment of sufficient conditions that have been recently proposed in the literature, see e.g. \cite{terzi2021learning} for LSTMs and \cite{bonassi2020stability, stipanovic2020stability} for GRUs. 
We then show that the stability of these systems allows one to guarantee the closed-loop stability of the IMC architecture, as well as to ease the  generation of the dataset used to learn the controller network itself.
}

% - No problema di determinare delay ingressi
% - Funziona per sistemi MIMO 
% - No problemi di improprieta'
% - Vincoli box su ingressi

% Results on numerical example
The proposed approach has been tested on the Quadruple Tank system described in \cite{alvarado2011comparative}.
In particular, after showing how a stable gated recurrent network can be used to learn the plant's model, we focus on the problem of generating a suitable dataset for the controller learning procedure.
\REV{The closed-loop performances have been eventually tested and compared to alternative control schemes. 
Results show that the proposed IMC architecture allows to obtain better performances than those obtained by a FFNN-based IMC, and in line with those of a MPC control law, although with a much lower online computational cost.}

\medskip
The paper is structured as follows. 
In Section \ref{sec:rnn} state-space RNNs are presented, with a particular focus on deep GRUs, for which recent stability results are summarized. 
Then, in Section \ref{sec:imc}, the IMC architecture is described, and we show how to use GRUs to  learn the model of the system and to approximate its inverse to obtain the controller.
Traditional stability results of IMC are summarized in Section \ref{sec:stability}.
In Section \ref{sec:example} the application of the proposed approach is described on the Quadruple Tank benchmark system, and the closed-loop performances are assessed. 
Eventually, conclusions are drawn in Section \ref{sec:conclusion}.

% While this simpler NN structure allows ensuring the existence of the model's inverse, it requires \emph{(i)} to carefully evaluate the model's delay; \emph{(ii)} to enforce the memory of the past reference and input signal by explicitly providing a suitably large amount of past data to the controller; \emph{(iii)} to introduce delays to ensure the controller's properness.

\subsection{Notation}
Given a vector $v$, we denote by $v^\prime$ its transpose and by $\| v \|_p$ its $p$-norm.  
Sequences of vectors are represented using boldface fonts, i.e. $\bm{v}_k = \{ v(0), v(1), ..., v(k) \}$. 
\REV{The $\ell_{p, q}$ norm of sequences is defined as $\| \bm{v}_k \|_{p, q} = \big\| \,[ \, \| v(0) \|_p, \| v(1) \|_p, ..., \| v(k) \|_p  ]^\prime \, \big\|_q$ where, in particular, $\| \bm{v}_k \|_{2, \infty} = \max_{t \in \{ 0, ... k \}} \| v(t) \|_2$.}
With reference to time-varying quantity $x(k)$, unless otherwise specified the generic time index $k$ may be omitted, i.e. $x = x(k)$, and superscript $^+$ may be used to indicate the same quantity at time $k+1$, i.e. $x^+ = x(k+1)$.
The symbol $\circ$ denotes the Hadamard (element-wise) product between matrices or vectors.
The sigmoid and tanh activation functions are indicated by $\sigma$ and $\phi$, respectively, i.e. $\sigma(x) = \frac{1}{1+e^{-x}}$ and $\phi(x) = \tanh(x)$.

\section{Recurrent neural networks} \label{sec:rnn}
This section aims to introduce recurrent networks and show how they can be employed to learn nonlinear dynamical systems. 
As discussed, since RNNs are stateful neural networks, they are well-suited to the system identification task since the memory of the past data is encoded in the states of the network.
In this work we consider gated recurrent network such as GRUs \cite{bonassi2020stability} and LSTMs \cite{terzi2021learning, bonassi2019lstm}.
In these architectures, the gates manage the flow of information throughout the network, thus avoiding the so-called vanishing and exploding gradient problems \cite{pascanu2013difficulty}.
A generic gated recurrent network in state-space form reads as
\begin{equation} \label{eq:rnn:system}
	\begin{dcases}
		\xi^+ = f(\xi, v; \Phi) \\
		\zeta = g(\xi; \Phi)
	\end{dcases},
\end{equation}
where $\xi \in \mathbb{R}^n$ is the state vector, $v \in \mathbb{R}^m$ is the input, $\zeta \in \mathbb{R}^p$ is the output, and $\Phi$ is the set of parameters called weights.
The goal of the learning procedure is to find the parametrization $\Phi$ which best explains the available data. 
For example, in a system identification task, one wants the RNN's output $\zeta$ to approximate the measured output trajectory $\bm{y}$ given the input sequence $\bm{u}$ applied to the plant.

For the sake of simplicity, the reminder of this work makes reference to GRU networks, but the approach can be easily extended to LSTM architectures \cite{terzi2021learning}.

\subsection{Deep GRU architecture}
\begin{subequations} \label{eq:rnn:gru}
In the following, the architecture of deep Gated Recurrent Units is described in detail.
Let us indicate by $M$ be the number of layers of the network. 
We denote by $\xi^l \in \mathbb{R}^{n_l}$ and $v^l \in \mathbb{R}^{m_l}$ the state and the input of the $l$-th layer, respectively.
The state equation of the layer $l \in \{1, ..., M\}$ reads as
\begin{equation} \label{eq:rnn:gru:state}
	\xi^{l,+} = z^l \circ \xi^l + (1-z^l) \circ \phi (W_r^l v^l + U_r^l \, f^l \circ \xi^l + b_r^l),
\end{equation}
where $z^l = z^l(\xi^l, v^l)$ and $f^l = f^l(\xi^l, v^l)$ are the update and forget gates, defined as
\begin{align} \label{eq:rnn:gru:gates}
	z^l(\xi^l, v^l) &= \sigma(W_z^l v^l + U_z^l \xi^l + b_z^l), \\
	f^l(\xi^l, v^l) &= \sigma(W_f^l v^l + U_f^l \xi^l + b_f^l).
\end{align}
$W_\star^l$, $U_\star^l$, and $b_\star^l$ are the weight matrices of the layer.
As far as $v^l$ is concerned, the input to the first layer is the network's input, i.e. $v^1 = v$, while the input of any following layer is the updated state of the previous one, which means
\begin{equation}\label{eq:rnn:gru:input}
	\begin{dcases}
		v^1 = v \\	
		v^l = \xi^{l-1, +} \quad \forall l \in \{2, ..., M\}
	\end{dcases}.
\end{equation}
Eventually, the following output transformation is applied
\begin{equation} \label{eq:rnn:gru:output}
	\zeta = \psi(U_o \xi^M + b_o),
\end{equation}
where $\psi$ may be a nonlinear activation function, such as $\sigma$ and $\phi$, or -- if \eqref{eq:rnn:gru:output} is linear -- the identity function \identity.
The set of weights $\Phi$ of the deep GRU is
\begin{equation*}\scalemath{0.95}{
    \Phi = \Big\{ \{ W_z^l, U_z^l, b_z^l, W_r^l, U_r^l, b_r^l, W_f^l, U_f^l, b_f^l \}_{\forall l \in \{1, ..., M\}}, U_o, b_o\Big\}}.
\end{equation*}
\end{subequations}

It is worth noticing that, collecting the states $\xi^l$ in a vector $\xi = [ \xi^{1 \prime}, ..., \xi^{M \prime}]^\prime$, the deep GRU model \eqref{eq:rnn:gru} falls in the nonlinear state-space form  \eqref{eq:rnn:system}.

\subsection{Stability properties} \label{sec:rnn:stability}
We now summarize few stability results concerning deep GRUs that will be useful in the remainder of the article.
For more details, the reader is addressed to \cite{bonassi2020stability}.
In the interests of clarity, in the following, we denote by $\xi(k, \bar{\xi}, \bm{v}_k)$ the state trajectory of the deep GRU at time $k$, obtained initializing \eqref{eq:rnn:gru} in the initial state $\bar{\xi}$ and applying the input sequence $\bm{v}_k$.

\begin{lemma}[Lemma 3, \cite{bonassi2020stability}]
    The set $\Xi = [-1, 1]^n = \bigtimes_{l=1}^M [-1, 1]^{n_l}$ is an invariant set of the deep GRU \eqref{eq:rnn:gru}, that is,
    \begin{equation*}
        \xi \in \Xi \,\, \Rightarrow  \,\, \xi^+ = f(\xi, v; \Phi) \in \Xi,
    \end{equation*}
    for any input $v$.
\end{lemma}
\begin{lemma} [Lemma 4, \cite{bonassi2020stability}]\label{lemma:invset}
	For any initial state $\bar{\xi} = [\bar{\xi}^{1 \prime}, ..., \bar{\xi}^{M \prime}]^\prime$ and any input sequence, it holds that
	\begin{enumerate}[i. \,]
	    \item if $\bar{\xi} \notin \Xi$, $\| \xi(k, \bar{\xi}, \bm{v}_{k}) \|_\infty$ is strictly decreasing with $k$, until $\xi(k, \bar{\xi}, \bm{v}_{k}) \in \Xi$;
	    \item the convergence happens in finite time, i.e. $\exists \bar{k} \geq 0$ finite such that, for any $k \geq \bar{k}$, $\xi(k, \bar{\xi}, \bm{v}_{k}) \in \Xi$;
	    \item the convergence of each component of $\xi$ into $[-1, 1]$ is exponential.
	\end{enumerate}
\end{lemma}

We remind that a function $\Psi(s)$ is of class $\mathcal{K}_\infty$ if it is monotonically increasing, $\Psi(0) = 0$, and $\Psi(s) \to \infty$ if $s \to \infty$. 
Similarly, $\Psi(s, t)$ is of class $\mathcal{KL}$ if it is $\mathcal{K}_\infty$ with respect to $s$ and $\Psi(s, t) \to 0$ when $t \to \infty$.
The following stability-like notions can hence be formalized.

\begin{definition}[GAS] \label{def:gas}
	System \eqref{eq:rnn:system} is said to be Globally Asymptotically Stable (GAS) if there exist $\beta \in \mathcal{KL}$ such that for any initial condition $\bar{\xi}$ and any time instant $k$
	\begin{equation} \label{eq:rnn:gas_definition}
		\| \xi(k, \bar{\xi}, 0) \|_2 \leq \beta(\| \bar{\xi} \|_2, k),
	\end{equation}
	where $\xi(k, \bar{\xi}, 0)$ is the state trajectory obtained initializing system \eqref{eq:rnn:system} in $\bar{\xi}$ and feeding it with the null input sequence.
\end{definition} 
Note that \eqref{eq:rnn:gas_definition} implies that, for any initial state, the autonomous system underlying \eqref{eq:rnn:system} converges to the origin.
\begin{definition}[$\delta$ISS]
	System \eqref{eq:rnn:system} is Incrementally Input-to-State Stable ($\delta$ISS) if there exists functions $\beta \in \mathcal{KL}$ and $\gamma_v \in \mathcal{K}_\infty$ such that for any pair of initial conditions $\bar{\xi}_a$ and $\bar{\xi}_b$, any pair of input sequences $\bm{v}_{a k} \in \mathcal \mathcal{V}$ and $\bm{v}_{b k} \in \mathcal \mathcal{V}$, and any time instant $k$
	\begin{equation} \label{eq:rnn:deltaiss_definition}
	    \| \xi_a(k, \bar{\xi}_a, \bm{v}_{a k}) - \xi_b(k, \bar{\xi}_b, \bm{v}_{b k} )\|_2 \leq \beta(\| \bar{\xi}_a - \bar{\xi}_b  \|_2, k) + \gamma_v(\| \bm{v}_{a k} - \bm{v}_{b k} \|_{2, \infty}),
	\end{equation}
	where $\xi_*(k, \bar{\xi}_*, \bm{v}_{*k})$ denotes the state of \eqref{eq:rnn:system}, initialized in $ \bar{\xi}_*$ and fed with $\bm{v}_{*k}$, at time $k$. 
	
\end{definition}
Remarkably, the $\delta$ISS property implies that the effects of initial conditions asymptotically vanish, so that the modeling performances are independent of the initialization. 
Moreover, the $\delta$ISS entails that the distance between the state trajectories produced by two different input sequences, $\bm{v}_{a k}$ and $\bm{v}_{b k}$, is bounded by a function $\gamma_v(\|\bm{v}_{a k} - \bm{v}_{b k} \|_{2,\infty})$, which is strictly increasing with the maximum distance between the two inputs.
Hence, the closer the input sequences, the closer the resulting state trajectories.
Lastly, it is worth noticing that if the system is $\delta$ISS it is also GAS.

In the following, we customarily assume that the input of the network is bounded in $V = [-1, 1]^m$, see \cite{goodfellow2016deep}. Then, the following Theorem from \cite{bonassi2020stability} provides a sufficient condition that guarantees the $\delta$ISS of the deep GRU, in terms of its parametrization $\Phi$.

\begin{theorem}[\cite{bonassi2020stability}, Corollary 2] \label{thm:deltaiss}
	The deep GRU network \eqref{eq:rnn:gru} is $\delta$ISS if, for each layer $l \in \{1, ..., M\}$, the weights satisfy the following condition 
	\begin{equation} \label{eq:rnn:deltaiss:condition}
		\| U_r^l \|_\infty \left( \frac{1}{4}  \| U_f^l \|_\infty + \bar{\sigma}_f^l \right) < 1 - \frac{1}{4} \frac{1 + \bar{\phi}_r^l}{1 - \bar{\sigma}_z^l} \| U_z^l \|_\infty,
	\end{equation}
	where
	\begin{subequations} \label{eq:single:deltaiss:bounds_relaxed}
	\begin{align}
		\bar{\sigma}_z^l &= \sigma( \| W_z^l \quad U_z^l \quad b_z^l \|_\infty ),\\
		\bar{\sigma}_f^l &= \sigma( \| W_f^l \quad U_f^l  \quad b_f^l \|_\infty ), \\
		\bar{\phi}_r^l &= \phi( \| W_r^l \quad U_r^l \quad b_r^l \|_\infty ).
	\end{align}
	\end{subequations}
\end{theorem}
\REV{Notably the $\delta$ISS condition reported in Theorem \ref{thm:deltaiss} can be enforced during the training of the network, so as to guarantee its stability, as discussed in the follower sections.}

\REV{
\begin{remark} \label{remark:stability}
For other RNN architectures, such as LSTMs, alternative stability definitions may be considered, such as the Lyapunov-like stability proposed in \cite{miller2018stable}. 
However, while -- similarly to $\delta$ISS -- the stability conditions devised in \cite{miller2018stable} can be enforced during the training procedure, they are limited to single-layer networks. 
\end{remark}}

\section{Internal model control architecture} \label{sec:imc}

\begin{figure}[t]
	\centering
	\includegraphics[width=0.6\linewidth]{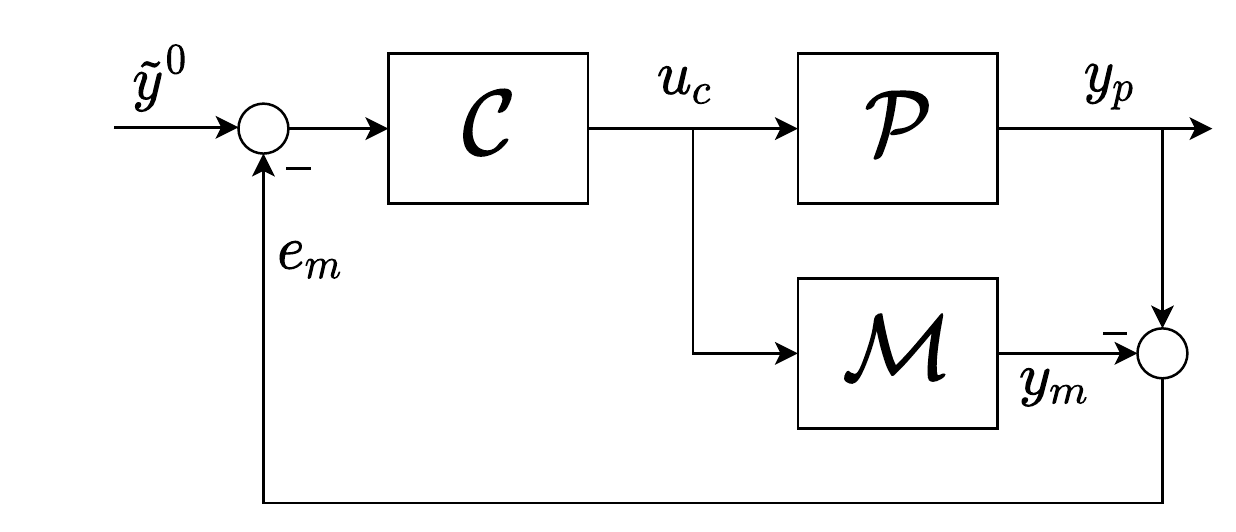}
	\caption{General scheme of IMC.}	
	\label{fig:imcscheme_noF}
\end{figure}

Having introduced GRUs and their stability properties, we can now discuss how they can be suitably employed in the Internal Model Control structure.
The IMC scheme, depicted in Figure \ref{fig:imcscheme_noF}, features three blocks: the unknown plant $\mathcal{P}$, its model $\mathcal{M}$, and the controller $\mathcal{C}$.
The plant $\mathcal{P}$ is assumed to be input-output stable, and it is described by the following (unknown) state-space dynamical system
\begin{equation} \label{eq:imc:model:system}
\mathcal{P}: \,
\begin{dcases}
    x^+ = f_p(x, u) \\
    y_p = g_p(x)
\end{dcases}.
\end{equation}
Let us denote by $y_p(k, \bar{x}, \bm{u}_{k})$ the output of \eqref{eq:imc:model:system}, initialized in $\bar{x}$ and fed with the input sequence $\bm{u}_{k}$.
\REV{For compactness we indicate its output trajectory as  $\bm{y}_{p,k}(\bar{x}, \bm{u}_{k}) = \big\{ y_p(0, \bar{x}, \bm{u}_{0}), \, ..., \,  y_p(k, \bar{x}, \bm{u}_{k}) \big\}$}\footnote{The same notation is adopted for the output trajectory of the model $\mathcal{M}$ and of the controller $\mathcal{C}$, i.e. 
\begin{gather*}
	\bm{y}_{m,k}(\bar{\xi}_m, \bm{u}_{k}) = \big\{ y_m(0, \bar{\xi}_m, \bm{u}_{0}), \, ..., \, y_m(k, \bar{\xi}_m, \bm{u}_{k}) \big\} \\
	\bm{u}_{c, k}(\bar{\xi}_c, \tilde{\bm{y}}^0_{k}) = \big\{ u_c(0, \bar{\xi}_c, \tilde{\bm{y}}^0_{0}), \, ..., \, u_c(k, \bar{\xi}_c, \tilde{\bm{y}}^0_{k}) \big\}
\end{gather*}}.

Ideally, one wants the system's model $\mathcal{M}$ to perfectly match $\mathcal{P}$ from an input-output perspective. 
This means that, letting $y_m(k, \bar{\xi}_m, \bm{u}_{k})$ be the output of the model $\mathcal{M}$ initialized in the state $\bar{\xi}_m$ and fed with the input sequence $\bm{u}_{k}$, then $\mathcal{M} \equiv \mathcal{P}$ if, for any input sequence $\bm{u}_{k}$ and any plant initial state $\bar{x}$, there exists an initial state of the model $\bar{\xi}_m$ such that $y_p(k, \bar{x}, \bm{u}_{k}) = y_m(k, \bar{\xi}_m, \bm{u}_{k})$, for any $k \geq 0$. \smallskip
% The model-plant mismatch feedback $e_m$ is null, at least asymptotically, when $\mathcal{M} \equiv \mathcal{P}$.

In the IMC paradigm, the controller block $\mathcal{C}$ is, ideally, the inverse of the model $\mathcal{M}$.
This implies that, for any output reference trajectory $\tilde{\bm{y}}^0_k$, and for any initial condition of the model $\bar{\xi}_m$, there exists an initial state of the controller $\bar{\xi}_c$ such that the control action $u_c$ generated by $\mathcal{C}$ steers the model output to the reference.
More formally, this condition reads as $\bm{y}_{m,k}(\bar{\xi}_m, \bm{u}_{c, k}) = \tilde{\bm{y}}^0$, where $\bm{u}_{c, k}$ denotes the sequence of control actions $u_c(k, \bar{\xi}_c, \tilde{\bm{y}}^0_{k})$ generated by the controller, i.e. the output of the controller $\mathcal{C}$ initialized in $\bar{\xi}_c$ and fed with the reference sequence $\tilde{\bm{y}}^0_{k}$.

In practice, however, $\mathcal{C}$ is synthesized as an approximation of the model inverse, since the exact inverse may be not proper, not analytically defined, or even not stable \cite{morari1989robust}.
Moreover, since a plant-model mismatch may be present, the IMC scheme depicted in Figure \ref{fig:imcscheme_noF} features the modeling error feedback $e_m=y_p - y_m$, see  \cite{morari1989robust}.

In the following we consider a model $\mathcal{M}$ learned by a GRU network, and we discuss how a GRU resembling its inverse can be trained and used as controller $\mathcal{C}$.
Moreover, in the reminder of this work we consider the modified control scheme shown in Figure~\ref{fig:imcscheme}, where a model reference block $\mathcal{M}_r$ has been added \cite{rivals2000nonlinear}. 
This block encodes the desired closed-loop response to the reference signal $y^0$.
Indeed, when the IMC control is perfect (i.e. $y_p = \tilde{y}^0$), the relationship between the reference signal $y^0$ and the system output $y_p$ is exactly the reference model $\mathcal{M}_r$. 

\REV{
\begin{remark} \label{remark:computational}
The IMC synthesis procedure described above takes place offline. Therefore, while training a recurrent network is a computationally-intensive task, it can be carried out on a sufficiently powerful workstation and then deployed to a control system with scarce computational resources. During online operations, the IMC control scheme boils down to the propagation of the model $\mathcal{M}$ and of the controller $\mathcal{C}$ based on the measured output and the filtered reference signal, which, as a sequence of tensor operations, can be efficiently done online.
\end{remark}}

\begin{figure}[t]
	\centering
	\includegraphics[width=0.7\linewidth]{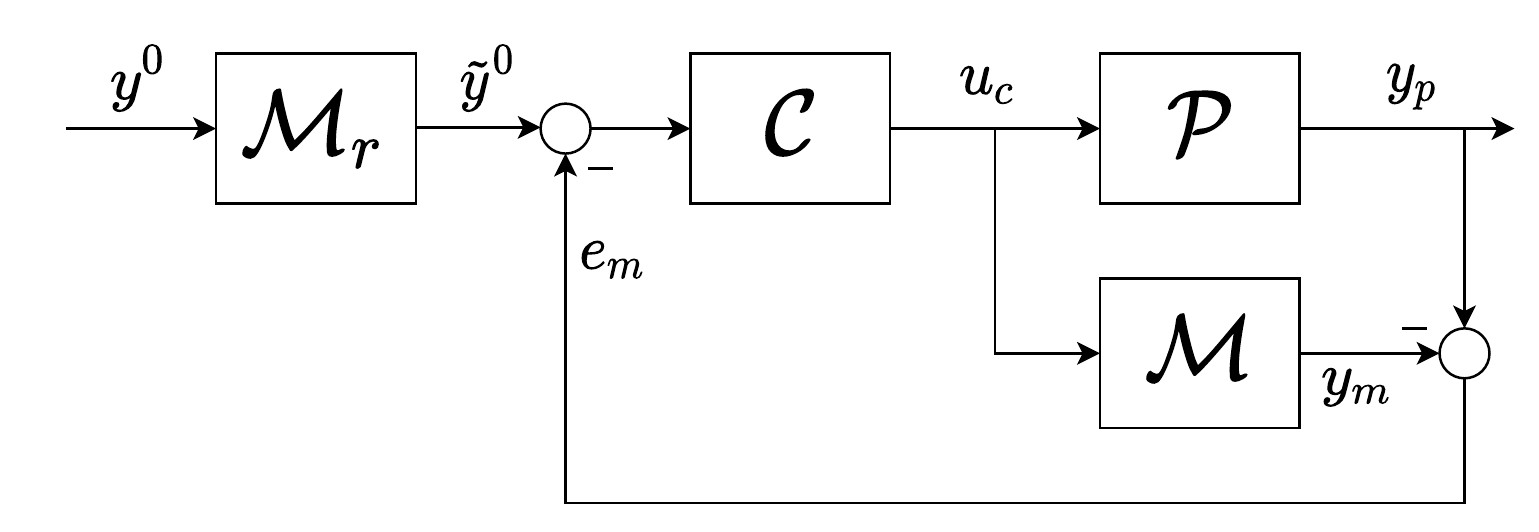}
	\caption{General scheme of IMC with model reference.}	
	\label{fig:imcscheme}
\end{figure}

\subsection{System model identification} \label{sec:imc:model}
\REV{As discussed, the first ingredient for the synthesis of the IMC architecture is the model of the system.}
\REV{In light of the performance of gated recurrent networks, we propose to use them to learn the system model, training the network to approximate the unknown plant from the input-output data collected from it.}
Such model network takes the following form
\begin{subequations} \label{eq:imc:model}
\begin{equation} \label{eq:imc:model:gru}
\mathcal{M}: \,
\begin{dcases}
    \xi_m^+ = f_m(\xi_m, u; \Phi_m) \\
    y_m = g_m(\xi_m; \Phi_m)
\end{dcases}.
\end{equation}
More specifically, the recurrent architecture here considered is a GRU in the form of \eqref{eq:rnn:gru}, with a linear output transformation, i.e. with $\psi = \text{\identity}$:
\begin{equation}
        g_m(\xi_m, \Phi_m) = U_{o, m} \, \xi_m^M  + b_{o,m}, 
\end{equation}
\end{subequations}
where $\xi_{m}^M$ is the state of the last layer.
The input of the model is the control action applied to the plant, $u$, and its output is $y_m$, desirably close to the plant's output $y_p$.
It is customarily assumed that $u$ and $y_m$ are normalized, i.e. they lie in the range $[-1, 1]$. \smallskip

\REV{The learning procedure consists in finding the parametrization $\Phi_m^*$ which minimizes the free-run simulation error $\bm{y}_{m, {\scriptscriptstyle T_s}}(\bar{\xi}_m, \bm{u}_{\scriptscriptstyle T_s}) - \bm{y}_{p, {\scriptscriptstyle T_s}}(\bar{x}, \bm{u}_{\scriptscriptstyle T_s})$}, ideally for any possible input sequence $\bm{u}_{\scriptscriptstyle T_s}$, $T_s$ being the simulation length.
In practice, one performs a limited number of experiments on the unknown plant, and collects $N_s$ input-output sequences from the unknown plant, denoted by $(\bm{u}^{\{i\}}_{\scriptscriptstyle T_s}, \bm{y}^{\{i\}}_{p, {\scriptscriptstyle T_s}})$,  $i \in \{1, ..., N_s\}$.
Note that each sequence $i$,  which has length $T_{s}^{\{i\}}$ \footnote{\REV{For notational simplicity, in the following it is assumed that the training sequences have the same length $T_{s}^{\{i\}}=T_s$.}}, may describe an experiment, or it may be obtained via the Truncated Back-Propagation Through Time (TBPTT) \cite{bianchi2017recurrent}.
In essence, TBPTT consists in extracting partially overlapping subsequences from a longer one, and it allows to significantly improve the performances of the trained network by artificially enlarging the dataset \cite{bianchi2017recurrent}. 

As discussed, the plant $\mathcal{P}$ is assumed to be input-output stable.  
Therefore, to ensure the consistency of the model to the plant, the $\delta$ISS of the model $\mathcal{M}$ is enforced.
We point out that the $\delta$ISS property discussed in Section \ref{sec:rnn} implies the input-output stability of the model \cite{khalil2002nonlinear}, since the output transformation is simply a static \REV{Lipschitz-continuous} transformation.
The condition stated in Theorem \ref{thm:deltaiss} is thus leveraged during the training procedure to ensure the stability of $\mathcal{M}$.
Since most training algorithms are unconstrained, condition \eqref{eq:rnn:deltaiss:condition} is relaxed by penalizing its violation in the loss function, as discussed in \cite{bonassi2020stability}.
Hence, at any iteration of the training algorithm, the loss function $L_m$ -- defined over a batch $\mathcal{I}$, that is, a random subset of sequences -- is minimized.
The loss function considered reads as
\begin{equation} \label{eq:imc:model:loss}
	L_m(\Phi_m) = \sum_{i \in \mathcal{I}} \text{MSE}(\bm{y}_{m, {\scriptscriptstyle T_s}}(\bar{\xi}_m, \bm{u}^{\{i\}}_{\scriptscriptstyle T_s}), \bm{y}^{\{i\}}_{p, {\scriptscriptstyle T_s}})  + \sum_{l=1}^M \rho(\nu^l(\Phi_m)),
\end{equation}
where
\begin{equation} \label{eq:imc:model:mse}
    \text{MSE}(\bm{y}_m, \bm{y}_p) = \frac{1}{T_s - T_{w}}  \sum_{k=T_{w}}^{T_s} \big\| y_m(k) - y_p(k) \big\|_2^2.
\end{equation}
The first term in $L_m$ corresponds to the Mean Square Error (MSE) between the the measured output sequence $\bm{y}^{\{i\}}_{p, {\scriptscriptstyle T_s}}$ and the model's open-loop prediction $\bm{y}_{m, {\scriptscriptstyle T_s}}(\bar{\xi}_m, \bm{u}^{\{i\}}_{\scriptscriptstyle T_s})$, obtained initializing \eqref{eq:imc:model:system} in the random initial state $\bar{\xi}_m$ and applying the input sequence $\bm{u}^{\{i\}}_{\scriptscriptstyle T_s}$.
Note that the output error is not penalized in the first $T_w$ steps, known as the washout period, to accommodate the initial transitory associated with the random initialization of the model. 
The second term of \eqref{eq:imc:model:loss} penalizes the violation of the stability condition \eqref{eq:rnn:deltaiss:condition} for each layer $l \in \{ 1, ..., M \}$, i.e. 
\begin{equation} \label{eq:imc:model:residual}
	\nu^l(\Phi_m) = \| U_r^l \|_\infty \left( \frac{1}{4}  \| U_f^l \|_\infty + \bar{\sigma}_f^l \right) + \frac{1}{4} \frac{1 + \bar{\phi}_r^l}{1 - \bar{\sigma}_z^l} \| U_z^l \|_\infty - 1,
\end{equation}
by means of a monotonically increasing function $\rho$, see \cite{bonassi2020stability}. 
Any training algorithm, such as SGD, Adam, or RMSProp, can be then used to train \eqref{eq:imc:model:system}, i.e. to retrieve a parametrization $\Phi_m^*$ for which the model is suitably accurate \cite{goodfellow2016deep}. 
In the following, we assume that the model $\mathcal{M}$ has been successfully trained and validated according to the proposed procedure, and that Theorem \ref{thm:deltaiss} holds, so that the model is stable.

\subsection{Controller learning} \label{sec:imc:controller}
Having an accurate model of the plant, according to the IMC paradigm, one should now find the right-inverse of the model, which, with a slight abuse of notation, will be denoted by $\mathcal{M}^{-1}$.
Due to the complexity of gated recurrent networks, retrieving an analytical expression of $\mathcal{M}^{-1}$ may not be possible, and its existence may not even be guaranteed. 

\begin{subequations} \label{eq:imc:controller:system}
For this reason, the following gated recurrent network, which approximates the model inverse, is considered
\begin{equation} 
    \mathcal{C}: \,
    \begin{dcases}
        \xi_c^+ = f_c(\xi_c, \tilde{y}^0; \Phi_c) \\
        u_c = g_c(\xi_c; \Phi_c)
    \end{dcases},
\end{equation}
where the input of the controller is the reference trajectory $\tilde{y}^0$ that should be tracked by the IMC scheme, obtained filtering the set-point $y^0$ with the model reference $\mathcal{M}_r$, see Figure \ref{fig:imcscheme}.
The state of the controller is $\xi_c$, and its output is the control signal applied to the system, denoted by $u_c$.
In this paper, we assume that also $\mathcal{C}$ is learned by a GRU network described by \eqref{eq:rnn:gru}, with output transformation
\begin{equation} \label{eq:imc:controller:system:output}
    g_c(\xi_c, \Phi_c) = \phi(U_{o, c} \, \xi_c^M + b_{o, c}),
\end{equation}
where $\xi_c^M$ is the state of the last GRU layer.
\end{subequations}
\REV{A notable property entailed by the structure of $\mathcal{C}$ is that it accounts for actuators' saturation constraints.}
Indeed, since the $\tanh$ activation function is used in the output transformation \eqref{eq:imc:controller:system:output}, the controller's output $u_c$ is guaranteed to lie in $[-1, 1]^m$.
This reflects the assumption that $u_c$ is normalized in such a way that it is unity-bounded, as discussed in Section \ref{sec:imc:model}. 
Other output activation functions, such as  saturations, can however be adopted. \smallskip

To train the controller, we adopt the learning procedure depicted in Figure \ref{fig:inverse},  similar to the procedures proposed in \cite{hunt1991neural, rivals2000nonlinear}.
The goal is to find the parametrization $\Phi_c$ such that the input signal $u_c$ generated by the network steers the model's output $y_m$ as close as possible to the reference $\tilde{y}^0$.
It is worth noticing that, at this stage, the system model has been already identified, and hence its weights $\Phi_m^*$ are fixed.

\begin{figure}[t]
	\centering
	\includegraphics[width=0.7\linewidth]{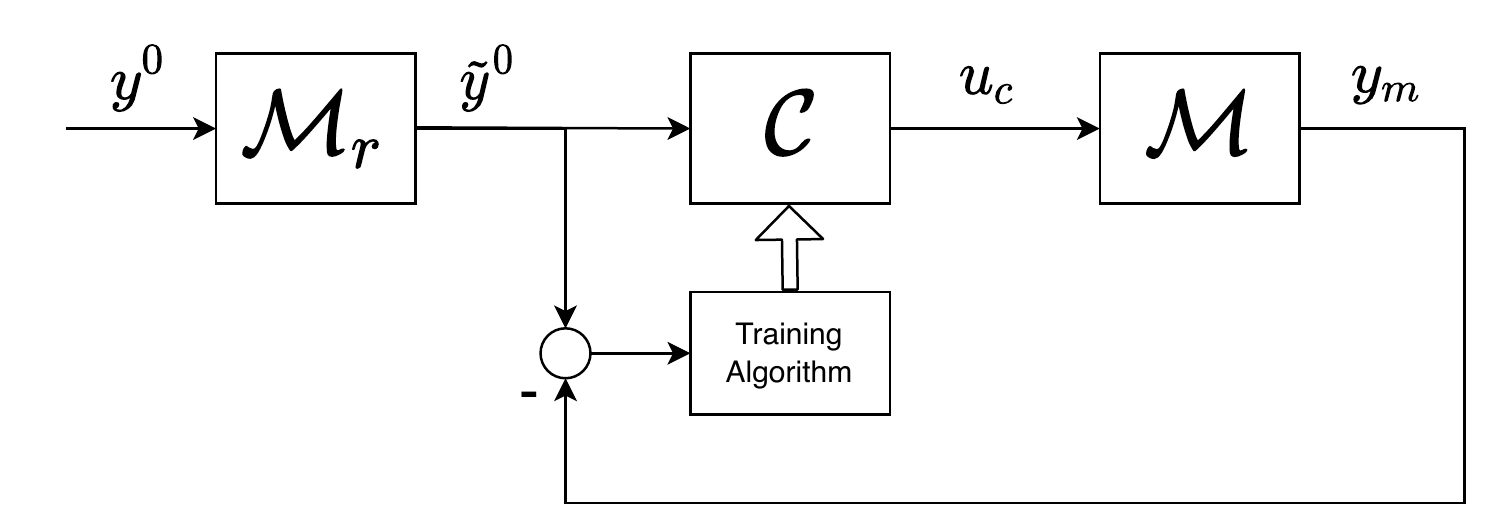}	
	\caption{Scheme of the controller's learning.}
	\label{fig:inverse}
\end{figure}

% Generazione del dataset arbitraria
The dataset used for the controller's training consists of a set of reference signals that the controller learns to track, and it is synthetically produced by generating a suitably large number of piece-wise constant references $y^0$, which are then filtered with the selected model reference $\mathcal{M}_r$.
The closer these references are to those imposed in closed-loop system operation, the more accurate the controller action will be.
Note that the generated references should be feasible for the model, meaning that, at least asymptotically, the difference between the filtered reference $\tilde{y}^0$ and the model output $y_m$ can be made sufficiently small through a suitable tuning of $\Phi_c$.
\REV{To this regard, denote by $\tilde{y}_s^0$ a generic steady-state value of the reference $\tilde{y}^0$.
Then, this feasibility condition boils down to require that the model admits an equilibrium characterized by $\tilde{y}_s^0$ as output}, i.e.  there exists a feasible input $u_s \in [-1, 1]^m$ and a state $\xi_s \in \Xi$ such that $\mathcal{M}$ admits $(u_s, \xi_s, \tilde{y}_s^0)$ as an equilibrium.
This crucial problem is further discussed in Section \ref{sec:example}. \smallskip

The training procedure of the controller $\mathcal{C}$ is carried out as follows.
First, an arbitrarily large amount of suitable reference signals, satisfying the aforementioned feasibility condition, are generated. 
These reference signals, denoted by $\bm{y}^{0, \{i\}}$, are then filtered with the model reference $\mathcal{M}_r$ to obtain $\tilde{\bm{y}}^{0, \{i\}}$.
These filtered references constitute the controller's training set.
At each training epoch, these sequences are  randomly divided in batches, and the loss function $L_c$ is minimized over each batch $\mathcal{I}$. 
The proposed loss function is 
\begin{equation}\label{eq:imc:controller:loss}
    L_c(\Phi_c) = \sum_{i \in \mathcal{I}} \text{MSE} \Big( \bm{y}_{m, {\scriptscriptstyle T_s}}(\bar{\xi}_m, \bm{u}_{c, {\scriptscriptstyle T_s}}(\bar{\xi}_c, \tilde{\bm{y}}^{0, \{i\}}_{\scriptscriptstyle T_s})), \, \tilde{\bm{y}}^{0, \{i\}}_{\scriptscriptstyle T_s} \Big) + \sum_{l=1}^M \rho(\nu^l(\Phi_c)).
\end{equation}
The first term is the mean square nominal output tracking error, that is, the MSE between the filtered output reference $\tilde{y}^0$ and the output of the model $\mathcal{M}$ controlled by the $\mathcal{C}$, see Figure \ref{fig:inverse}.
The second term, as discussed in Section \ref{sec:imc:model}, allows to fulfill the controller's $\delta$ISS condition \eqref{eq:rnn:deltaiss:condition}.
\REV{Any training algorithm, such as Adam or RMSProp, can be used to train $\mathcal{C}$ by minimizing $L_c$}, thus retrieving the controller's weights $\Phi_c^*$ that make the nominal output tracking error over the validation dataset sufficiently small.

At this stage, the control system has been entirely learned from the data. 
In the following Section the stability properties of the proposed control scheme are discussed. 

\section{Stability properties} \label{sec:stability}
Owing to its particular structure, the IMC scheme enjoys the following closed-loop properties \cite{hunt1991neural}.

\begin{property}[Stability \cite{economou1986internal}] \label{property:stability}
If the plant and the controller are input-output stable and the model is exact, the closed-loop system is input-output stable. 
\end{property}
Note that, in absence of output noise, if $\mathcal{M} = \mathcal{P}$ the modeling error feedback is null, i.e. $e_m = 0$, and the IMC controller operates in open-loop. 
Thus, if $\mathcal{C}$ and $\mathcal{P}$ are input-output stable, the overall scheme is input-output stable as well.
A $\delta$ISS controller, trained as discussed in Section \ref{sec:imc:controller}, is hence able to guarantee the closed-loop stability.

\begin{property} [Perfect Control \cite{economou1986internal}] \label{property:perfect_control}
Assume that the plant is input-output stable, that the model is exact, and that the controller matches the model's inverse, i.e. $\mathcal{C} = \mathcal{M}^{-1}$. Then, if $\mathcal{C}$ is input-output stable, the closed-loop is input-output stable and matches the model reference $\mathcal{M}_r$.
\end{property}
Indeed, under the assumption of exact model availability, the control system operates in open-loop.
Since $\mathcal{C} = \mathcal{\mathcal{M}}^{-1} = \mathcal{P}^{-1}$, it holds that $y_p = \tilde{y}^0$, and hence the relationship between $y^0$ and $y_p$ is the model reference $\mathcal{M}_r$.

\begin{property} [Zero Offset \cite{economou1986internal}] \label{property:offset_free}
Assume that the plant is input-output stable, that the model is exact and admits an inverse, and that the steady state control action generated by the controller matches the steady-state value of the model's inverse.
Then, if the controller is input-output stable, offset-free tracking is asymptotically attained.
\end{property}
This last property means that, owing to the stability of $\mathcal{M}$ and $\mathcal{C}$, if the controller matches the inverse operator of the model -- at least at steady state --, then $y_m(k, \bar{\xi}_m, \bm{u}_{c, k}( \bar{\xi}_c, \tilde{\bm{y}}^{0}_{k})) \xrightarrow[k \to \infty]{} \tilde{y}^0(k)$, i.e. offset-free control is achieved.

While these properties are remarkable, in practice it is very hard to guarantee the absence of a plant-model mismatch and the exactness of the model's inverse.
In the following, we thus show the closed-loop input-output stability in non-ideal cases, where $\mathcal{M}$ is not exact and $\mathcal{C}$ does not match the model's inverse.

\begin{proposition}
Assume that the plant $\mathcal{P}$ \eqref{eq:imc:model:system} can be described by the equations of the model $\mathcal{M}$ with a fictitious additive disturbance $d$, which accounts for the plant-model mismatch
\begin{equation}
    \mathcal{P}: \, \begin{dcases}
        x^+ = f_m(x, u; \Phi_m^*) \\
        y_p = g_m(x; \Phi_m^*) + d
    \end{dcases}.
\end{equation}
Then, if the model $\mathcal{M}$ and the controller $\mathcal{C}$ are $\delta$ISS, the closed-loop IMC scheme depicted in Figure \ref{fig:imcscheme} is input-output stable with respect to the reference trajectory $\tilde{\bm{y}}^0$.
\end{proposition}
\begin{proof}
    Owing to Lemma \ref{lemma:invset}, and in light of the output transformation of $\mathcal{C}$ defined in \eqref{eq:imc:controller:system}, the control variable $u$ is unity-bounded, i.e. $u(k) \in [-1, 1]^m$.
    Therefore, since both $\mathcal{P}$ and $\mathcal{M}$ are stable, then $y_p$ and $y_m$ are bounded, which implies that $d$ is bounded as well. Hence the system is input-output stable.
\end{proof}
It is worth noticing that the modeling error feedback $e_m$, defined as
\begin{equation}
    e_m(k) = y_p(k, \bar{x}, \bm{u}_k) -   y_m(k, \bar{\xi}_m, \bm{u}_k),
\end{equation}
converges to the bounded disturbance $d$ that represents the plant-model mismatch, since the effect of the different initial conditions ($\bar{x}$ and $ \bar{\xi}_m$) is guaranteed to asymptotically vanish by the $\delta$ISS of $\mathcal{M}$.
% Since both $\mathcal{M}$ and $\mathcal{P}$ are input-output stable, $d$ is finite.

\section{Numerical example} \label{sec:example}
\subsection{Benchmark description}
\begin{figure}
    \centering
    \includegraphics[width=0.6 \linewidth]{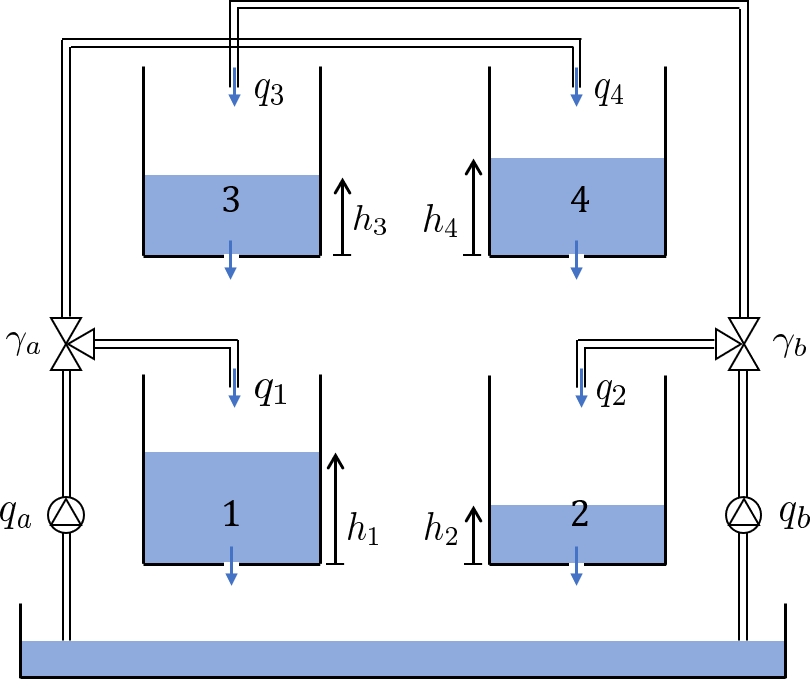}
    \caption{Quadruple tank system \cite{bonassi2020stability}}
    \label{fig:qt}
\end{figure}
The performances of the proposed control scheme have been assessed on the Quadruple Tank system reported in \cite{alvarado2011comparative}.
The system, depicted in Figure \ref{fig:qt}, consists in four tanks containing water, whose levels are denoted by $h_1$, $h_2$, $h_3$, and $h_4$, which are fed with two controllable pumps.
Specifically, two triple valves split the flow rate $q_a$ in $q_1 = \gamma_a q_a$ and $q_3 = (1 - \gamma_a) q_a$, and the flow rate $q_b$ in $q_2 = \gamma_b q_b$ and $q_4 = (1 - \gamma_b) q_b$.
The equations of the system are
\begin{subequations} \label{eq:example:qt}
\begin{equation}
    \begin{aligned}
    \dot{h}_1 &= - \frac{a_1}{S} \sqrt{2gh_1} + \frac{a_3}{S} \sqrt{2gh_3} + \frac{\gamma_a}{S} q_a, \\
	\dot{h}_2 &= - \frac{a_2}{S} \sqrt{2gh_2} + \frac{a_4}{S} \sqrt{2gh_4} + \frac{\gamma_b}{S} q_b, \\
	\dot{h}_3 &= - \frac{a_3}{S} \sqrt{2gh_3} + \frac{1 - \gamma_b}{S} q_b, \\
	\dot{h}_4 &= - \frac{a_4}{S} \sqrt{2gh_4} + \frac{1 - \gamma_a}{S} q_a,
    \end{aligned}	
\end{equation}
where the parameters of the system have been reported in Table \ref{tab:qt_parameters}. 
The water levels, as well as the control variables, are also subject to saturation limits
\begin{equation} \label{eq:example:qt:bounds}
\begin{aligned}
    h_i &\in [h_i^{min}, h_i^{max}] \quad \forall i \in \{ 1, ..., 4 \},  \\
    q_a &\in [q_a^{min}, q_a^{max}], \\
    q_b &\in [q_b^{min}, q_b^{max}].
\end{aligned}
\end{equation}
\end{subequations}
In the following it is assumed that only $h_1$ and $h_2$ are measurable, i.e. the output of the system is $y_p = [h_1, h_2]^\prime$, while the input of the system is $u = [q_a, q_b]^\prime$.
The control goal is to steer the system's output $y_p$ to the reference $y^0$ mimicking the response of the reference model $\mathcal{M}_r$.

As discussed in Section \ref{sec:imc}, the synthesis of an IMC regulator is articulated in the following steps: \emph{(i)} learning a model $\mathcal{M}$ of the system; \emph{(ii)} generating a dataset of feasible reference trajectories; \emph{(iii)} learning a controller $\mathcal{C}$ which approximates the model's inverse. 
\REV{In the following subsections, these three steps are tackled.}

\begin{table}
	\centering
	\caption{Benchmark system parameters}
	\label{tab:qt_parameters}
		\begin{tabular}{cccc|cccc}
		\toprule
		Parameter & Value & Units & $\,\,\,$ & $\,\,\,$ & Parameter & Value & Units \\
		\midrule 
		$a_1$ & $1.31 \cdot 10^{-4}$ & $\text{m}^2$  &&& $[h_1^{min}, h_1^{max}]$ & $[0, 1.36]$ & m \\
		$a_2$ & $1.51 \cdot 10^{-4}$ & $\text{m}^2$  &&& $[h_2^{min}, h_2^{max}]$ & $[0, 1.36]$ & m \\
		$a_3$ & $9.27 \cdot 10^{-5}$ & $\text{m}^2$  &&& $[h_3^{min}, h_3^{max}]$ & $[0, 1.3]$ & m \\
		$a_4$ & $8.82 \cdot 10^{-5}$ & $\text{m}^2$  &&& $[h_4^{min}, h_4^{max}]$ & $[0, 1.3]$ & m \\
		$S$ & $0.06$ & $\text{m}^2$ &&& $[q_a^{min}, q_a^{max}]$ & $[0, \, 9\cdot 10^{-4}]$ & $\frac{\text{m}^3}{s}$\\
		$\gamma_a$ & $0.3$ &  &&& $[q_b^{min}, q_b^{max}]$ & $[0, \, 1.3\cdot 10^{-3}]$ & $\frac{\text{m}^3}{s}$\\
		$\gamma_b$ & $0.4$ & &&& & & \\
		\bottomrule
		\end{tabular}
\end{table}

\subsection{Model training}
The Quadruple Tank system described by \eqref{eq:example:qt} has been implemented in MATLAB.
In order to retrieve the data required for the model's training, the system has been fed with Multilevel Pseudo-Random Signals (MPRS) as inputs to properly excite the system and collect data in a broad operating region.
According to the TBPTT paradigm, $N_s = 200$ random partially-overlapping sequences have been extracted from the experiment.
Each pair of input-output sequences $(\bm{u}^{\{i\}}_{\scriptscriptstyle T_s}, \bm{y}^{\{i\}}_{p, {\scriptscriptstyle T_s}})$ is made by $T_s =700$ data-points, collected with a sampling time $\tau_s = 25s$.
The data has been suitably normalized, so that \eqref{eq:example:qt:bounds} translates into the unity-boundedness of $y_p$ and $u$.

A deep GRU with $M=2$ layers, made by $n_l=10$ units each, has been used to learn the system model.
The training procedure discussed in Section \ref{sec:imc:model} has been carried out with TensorFlow 1.15 on Python 3.7, using RMSProp \cite{goodfellow2016deep} to minimize the loss function $L_m(\Phi_m)$ defined in \eqref{eq:imc:model:loss}, thus retrieving $\Phi_m^*$.
It is worth recalling that the $\delta$ISS property  has been enforced by penalizing the violation of the $\delta$ISS condition \eqref{eq:imc:model:residual}.
This is achieved using a piece-wise linear cost $\rho(\nu^l(\Phi_m))$ which steers $\nu^l(\Phi_m)$, $l \in \{ 1, ..., M \}$, to some sufficiently small negative value, see \cite{bonassi2020stability}.
\\

At each training epoch, the set of $N_s$ sequences is randomly split into several batches, with respect to which the optimizer tunes the network's parameters.
After each epoch, the performances of the network have been evaluated on $25$ validation sequences extracted from an independent experiment. 
The network training was halted when the model's performance on the validation set stopped improving, thus obtaining the model's weights $\Phi_m^*$.

\begin{figure}[ht!]
	\centering
	\includegraphics[width=0.8 \linewidth]{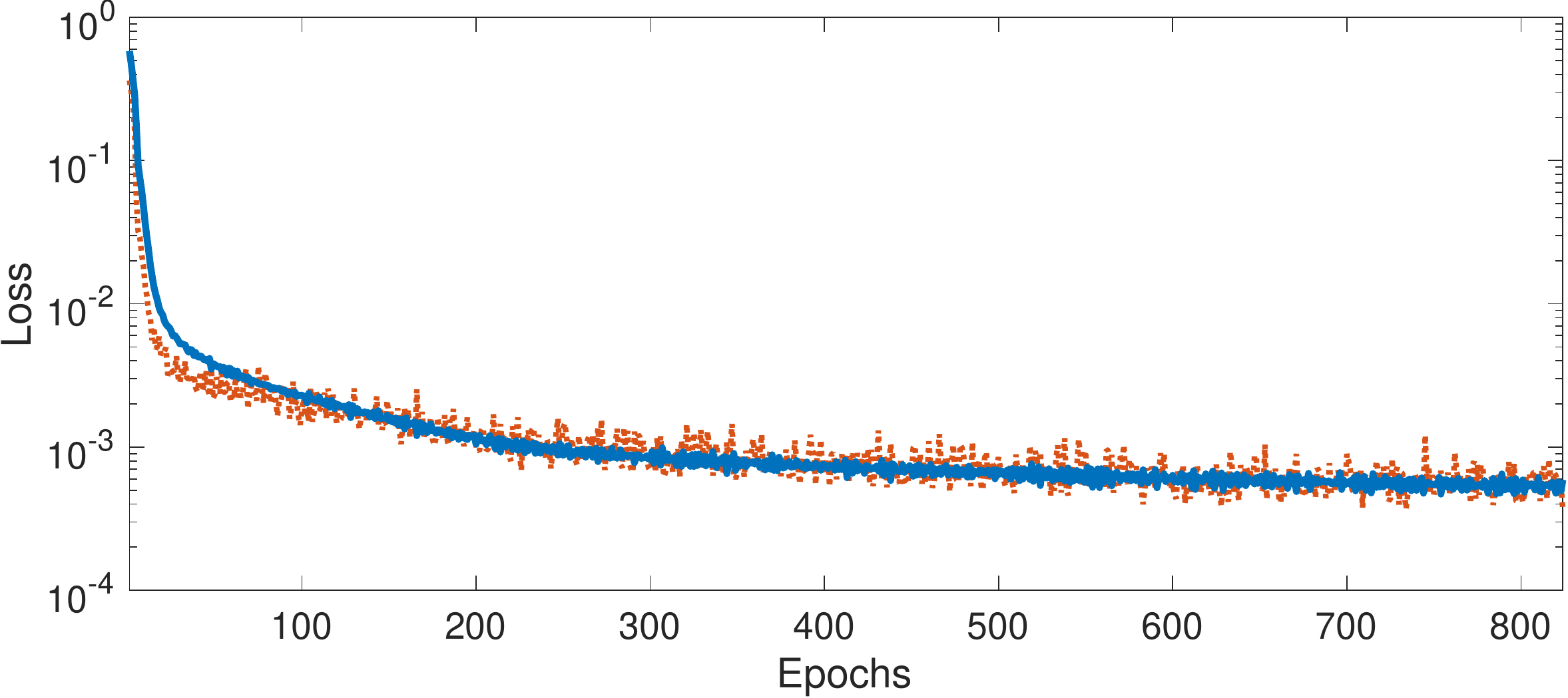}
	\caption{Model training: evolution of the average loss function across all the batches (blue line) and MSE on the validation dataset (red dotted line).}
	\label{fig:example:model_loss}
	\vspace{2mm}
	\centering
	\includegraphics[width=0.8 \linewidth]{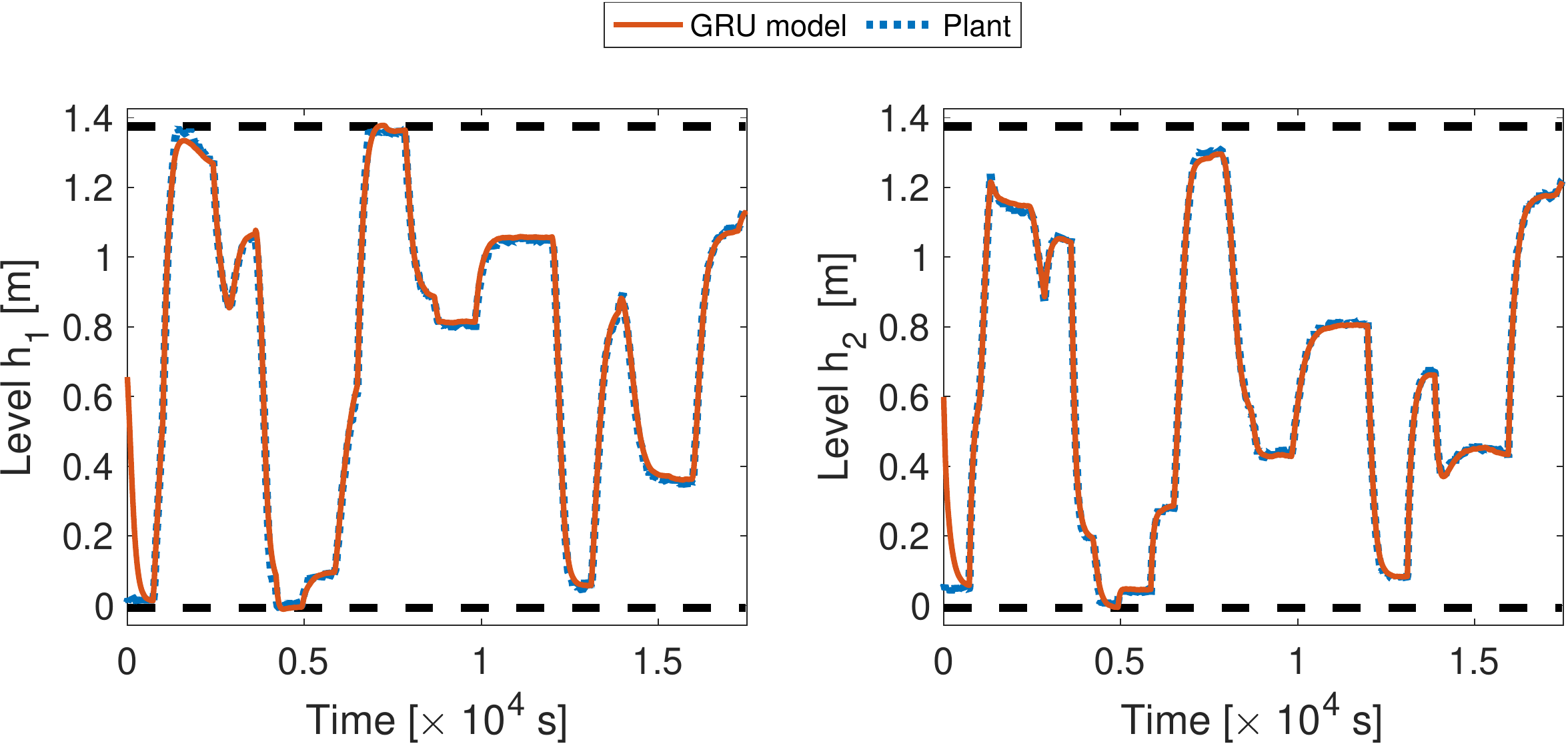}
	\caption{Performances of the trained model on an independent test set: GRU open-loop prediction (red line) compared to the plant's output (blue dotted line). Level $h_1$ is displayed on the left, level $h_2$ on the right.}
	\label{fig:example:model_test}
\end{figure}

The evolution of the loss function is shown in Figure \ref{fig:example:model_loss}. 
Overall, the training procedure took $823$ epochs.
Eventually, the modeling performances have been tested on an independent test sequence, leading to the satisfactory results depicted in Figure \ref{fig:example:model_test}.
These performances have also been quantified using the so-called FIT index $[ \% ]$, defined as
\begin{equation}
    \text{FIT} = 100 \left( 1 - 
    \sqrt{\frac{\text{MSE}(\bm{y}_{m, {\scriptscriptstyle T_s}} (\bar{\xi}_m, \bm{u}^{\{ts\}}_{\scriptscriptstyle T_s}), \bm{y}^{\{ ts \}}_{p, {\scriptscriptstyle T_s}})}{\text{MSE}(\bm{y}^{\{ ts \}}_{p, {\scriptscriptstyle T_s}}, y^{\{ts\}}_{avg})}} \, \right),
\end{equation}
where $(\bm{u}^{\{ts\}}_{\scriptscriptstyle T_s}, \bm{y}^{\{ts\}}_{p, {\scriptscriptstyle T_s}})$ is the test sequence and $y^{\{ts\}}_{avg}$ the output average value.
The trained model scores $FIT = 96.5 \%$, which indicates remarkable modeling performances.

\subsection{Feasible set-points generation}
After learning the model, the dataset used to train the controller network $\mathcal{C}$ has been generated.
As discussed in Section \ref{sec:imc:controller}, this dataset consists in a set of reference trajectories that the controller should learn to track -- no data needs to be collected from the real system at this stage.
These reference trajectories are rather obtained by filtering MPRB signals, denoted by $\bm{y}^0_{\scriptscriptstyle T_s}$, with the reference model $\mathcal{M}_r$ herein chosen as the discrete-time equivalent of a pair of decoupled first-order systems with unitary static gain and time constant $\tau_r = 2000s$.

% Generation of feasible set-points
For successful controller training, the dataset must be properly generated.
Indeed, the inclusion of unfeasible set-points $y^0(k)$, i.e. set-points which (given the input constraints) do not correspond to any feasible equilibrium of the model, would inevitably alter the loss function's gradient,  leading to poor performances of the trained controller. 

\begin{figure}[t]
	\centering
	\includegraphics[width=0.6\linewidth, clip, trim=0cm 0mm 0cm 5mm]{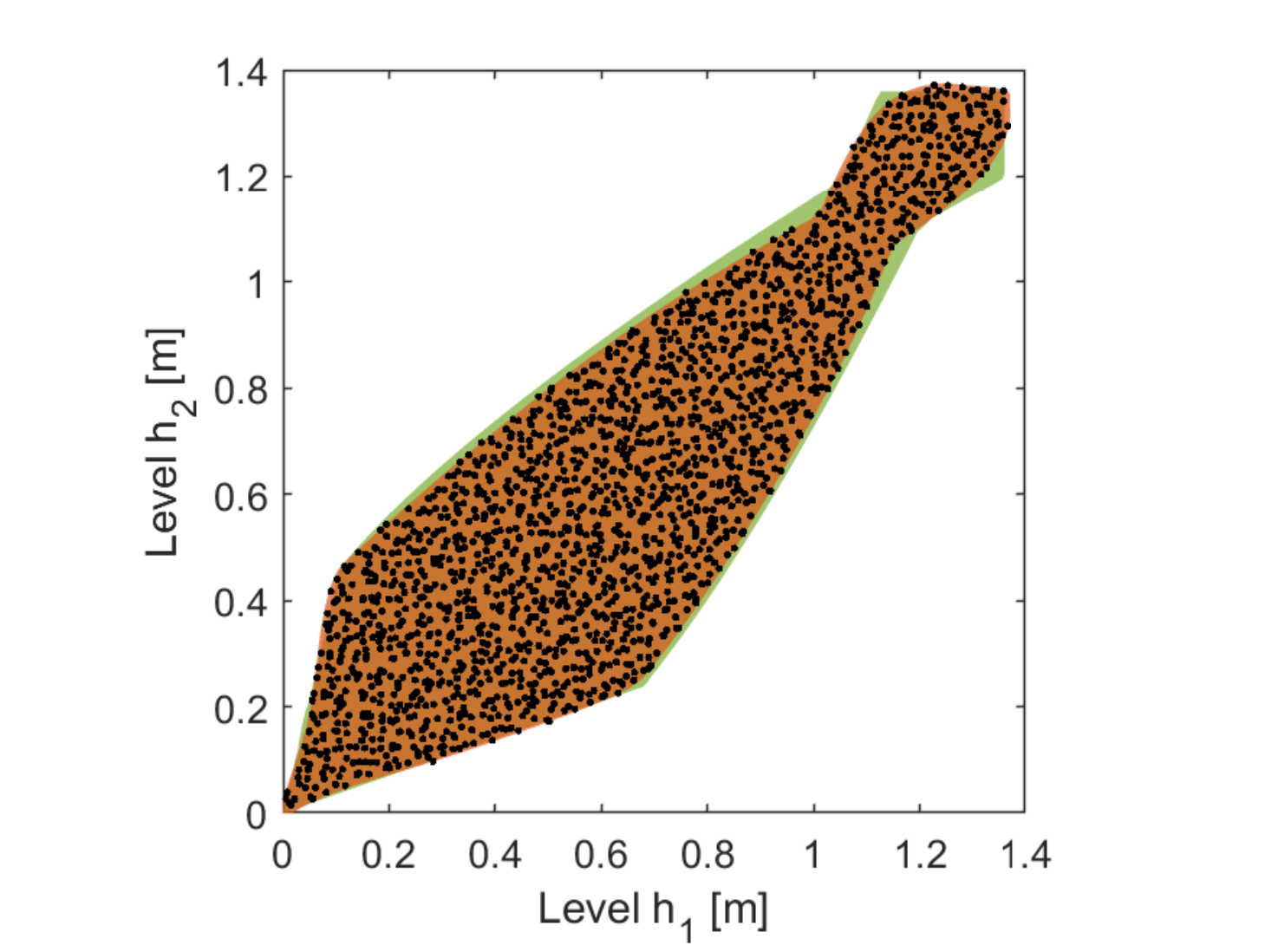}
	\caption{Controller dataset generation: set-points (black dots) are extracted from the set of model's feasible outpu set-points (orange area); the green area corresponds to the set of plant's feasible outputs, generally unknown.}
	\label{fig:example:controller_dataset_ss}
\end{figure}

In light of the model's $\delta$ISS, checking the existence of an input sequence that steers the model to the set-point reduces to assess the existence of some feasible constant input $u_s$ and some state $\xi_s \in \Xi$ such that $(u_s, \xi_s, y^0)$ is an equilibrium of the model, i.e.
\begin{equation} \label{eq:example:controller:eq_search}
\begin{dcases}
    \xi_s = f_m(\xi_s, u_s; \Phi_m^*) \\
    y^0 = g_m(\xi_s; \Phi_m^*)
\end{dcases}.
\end{equation}
Moreover, the $\delta$ISS property guarantees that, if such $u_s$ exists, it is unique and that the equilibrium $(u_s, \xi_s, y^0)$ can be reached from any initial state of the model.
Hence, this property further allows to relieve the computational complexity of solving \eqref{eq:example:controller:eq_search}, since initial guesses of $u_s$ and $\xi_s$ -- to be used for warm-starting the nonlinear solver -- can be easily retrieved from open-loop simulations of \eqref{eq:imc:model:system}.

Overall, $N_s = 430$ reference trajectories $\bm{y}^{0, \{ i \}}_{\scriptscriptstyle T_s}$  have been generated, $380$ of which used for training, $40$ used for validation, and $10$ for the final testing. 
Each reference trajectory consists in sequences of random steps which satisfy the aforementioned feasibility condition.
In Figure \ref{fig:example:controller_dataset_ss} the extracted random set-points are compared to the set of feasible model outputs, and to the set of feasible outputs of the real plant.
This latter is generally unknown, and is here reported just for comparison purposes.
The MPRB references are then filtered with the model reference $\mathcal{M}_r$ to obtain the reference trajectories $\tilde{\bm{y}}^{0, \{i\}}_{\scriptscriptstyle T_s}$, used for the training of the controller GRU network.

\begin{figure}[ht!]
	\centering
	\includegraphics[width=0.8\linewidth]{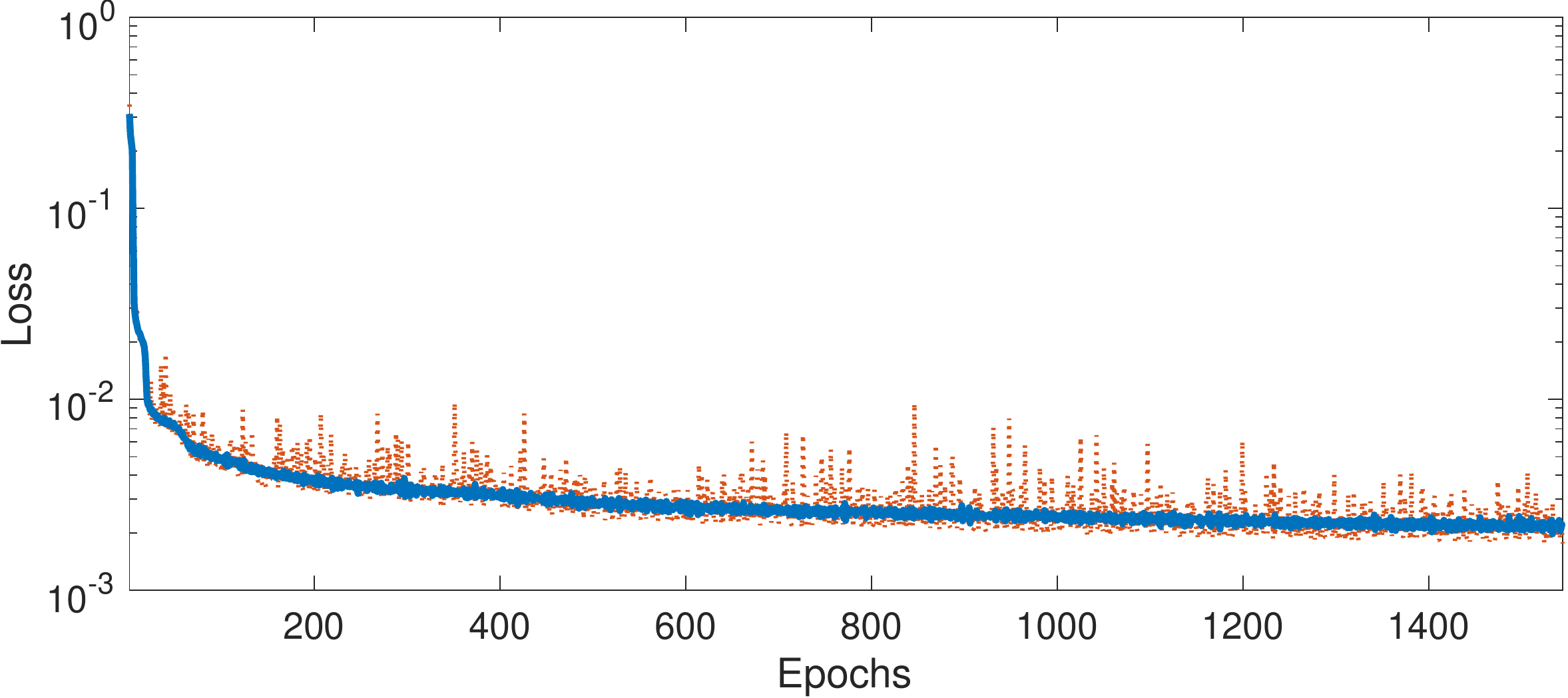}
	\caption{Controller training: evolution of the average loss function across all the batches (blue line) and MSE on the validation dataset (red dotted line).}
	\label{fig:example:controller_learning}
	\vspace{2mm}
	\centering
	\includegraphics[width=0.8\linewidth]{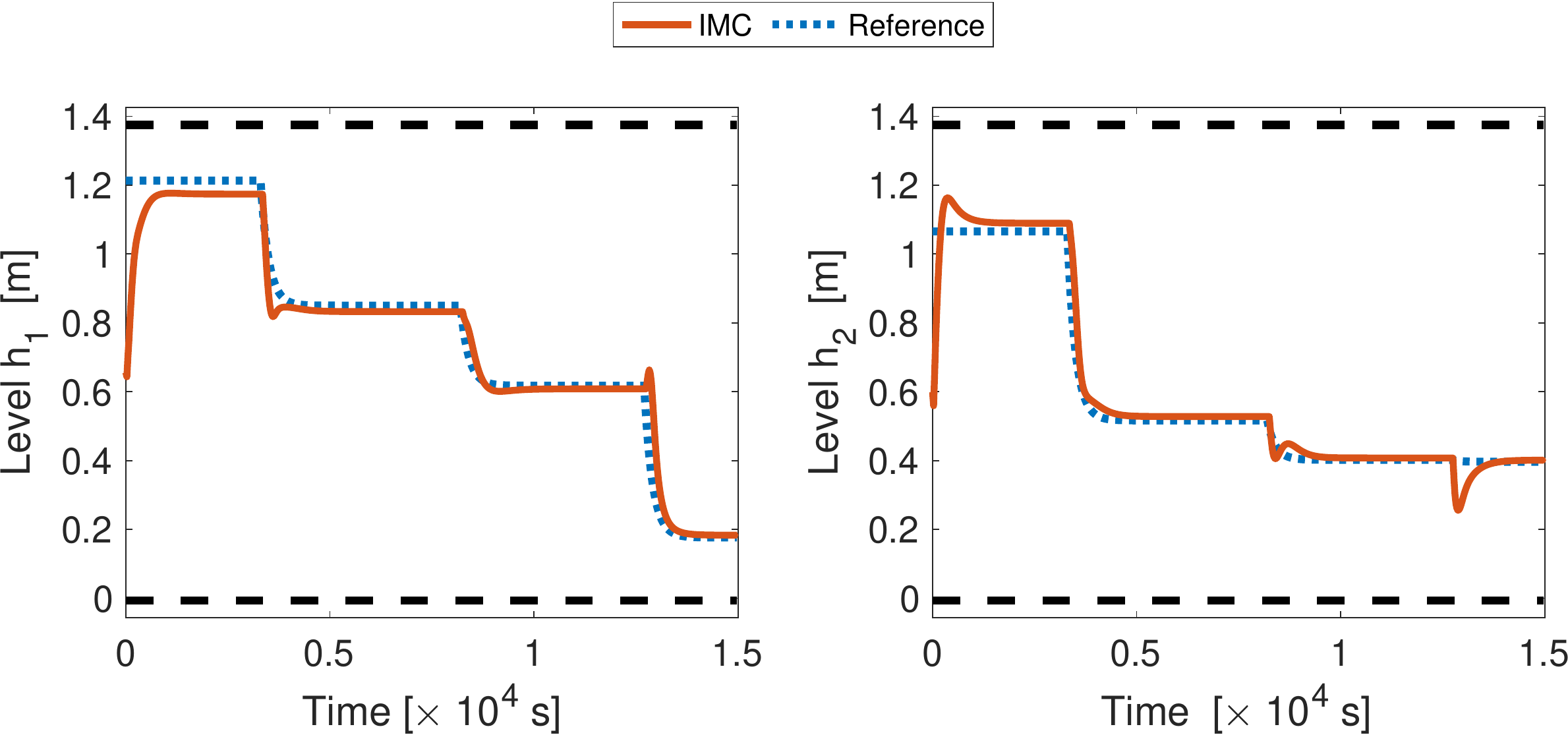}
	\caption{Performances of the controller on a reference trajectory extracted from the independent test set: output reference (blue dotted line) versus IMC-controlled output (red line). The level $h_1$ is displayed on the left, the level $h_2$ on the right.}
	\label{fig:example:controller_test}
\end{figure}

%\begin{figure}[t]
%	\centering
%	\includegraphics[width=0.6 \linewidth]{IMC_full.pdf}
%	\caption{IMC architecture with filtered modeling error feedback.}
%	\label{fig:imcscheme_filter}
%\end{figure}

\subsection{Controller training}
% Controller structure
The controller is learned by a deep GRU with $M=3$ layers, each one featuring $n_l = 5$ units.
The training procedure proposed in Section \ref{sec:imc:controller} has been carried out with TensorFlow 1.15, using RMSProp as optimizer.
As for the training of the model, the $\delta$ISS property of the controller network has been enforced using a piece-wise linear cost $\rho(\nu^l)$ which penalizes the violation of the $\delta$ISS condition \eqref{eq:rnn:deltaiss:condition}.

At each epoch of the training procedure, the set of references is randomly divided into batches, with respect to which the optimizer tunes the weights of the controller network so as to minimize the loss function $L_c(\Phi_c)$ \eqref{eq:imc:controller:loss}.
After each epoch, the performance metrics have been evaluated on the validation set, and the training was halted when they stop improving, obtaining $\Phi_c^*$.
The evolution of the loss function throughout the controller training procedure, which took $1543$ epochs, is shown in Figure \ref{fig:example:controller_learning}.

% Training results
Eventually, the controller's performances have been tested on the independent test-set's reference trajectories.
The controller's open-loop performances are depicted in Figure \ref{fig:example:controller_test}, where the IMC-controlled output $\bm{y}_{m, {\scriptscriptstyle T_s}}(\bar{\xi}_m, \bm{u}_{c, {\scriptscriptstyle T_s}}(\bar{\xi}_c, \tilde{\bm{y}}^{0, \{i\}}_{\scriptscriptstyle T_s}))$ is compared to its reference $\tilde{\bm{y}}^{0, \{i\}}_{\scriptscriptstyle T_s}$.
Note that, in addition to limited steady-state tracking errors, cross-coupling effects between the two outputs are present.
These effects, albeit quickly compensated, may deteriorate the FIT index, yet a remarkable $FIT = 87 \%$ is scored for the reported test reference trajectory.

\subsection{Closed-loop performances} \label{sec:results}
Eventually, the closed-loop performances of the proposed control architecture have been tested on the simulated Quadruple Tank system, where the outputs have been corrupted by a white Gaussian noise with standard deviation $0.01$.

\begin{figure}[t]
    \centering
	\includegraphics[width=0.7 \linewidth]{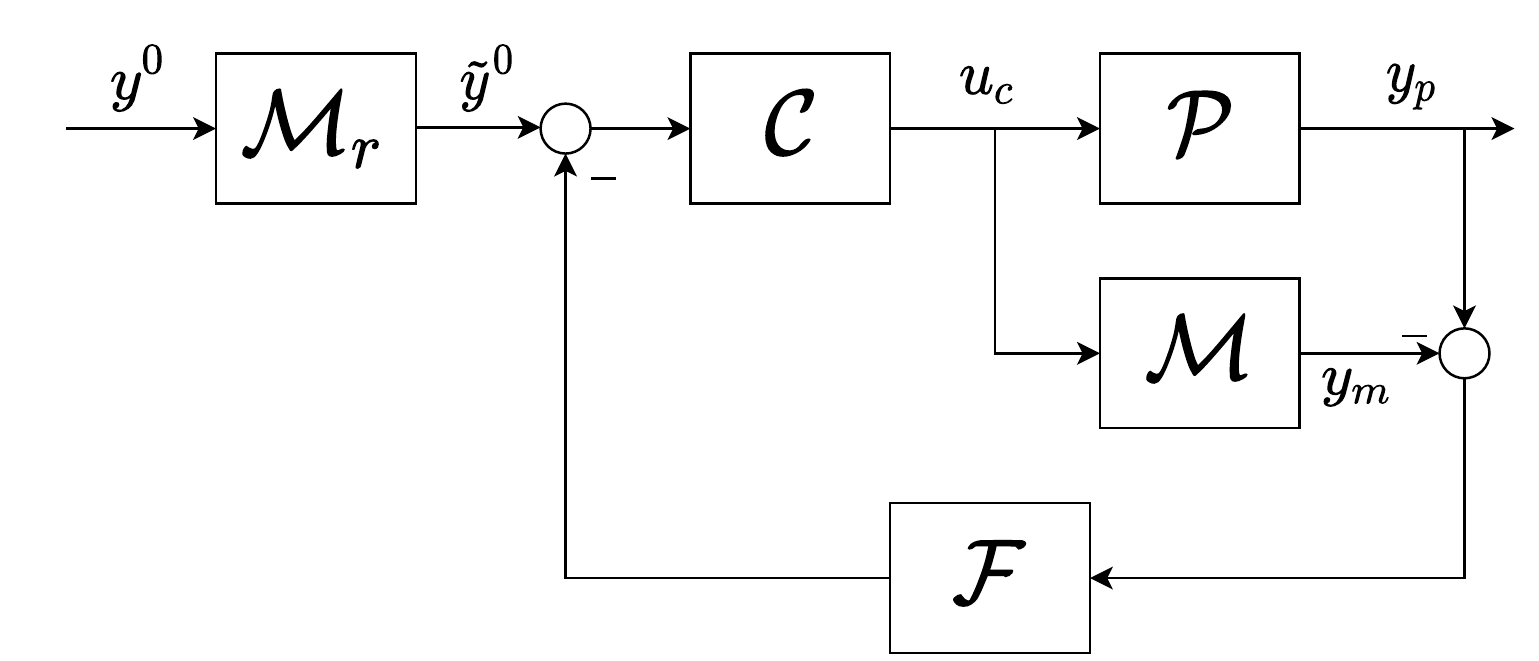}
	\caption{IMC architecture with filtered modeling error feedback.}
	\label{fig:imcscheme_filter}
	\vspace{2mm}
	\centering
	\includegraphics[width=0.8 \linewidth]{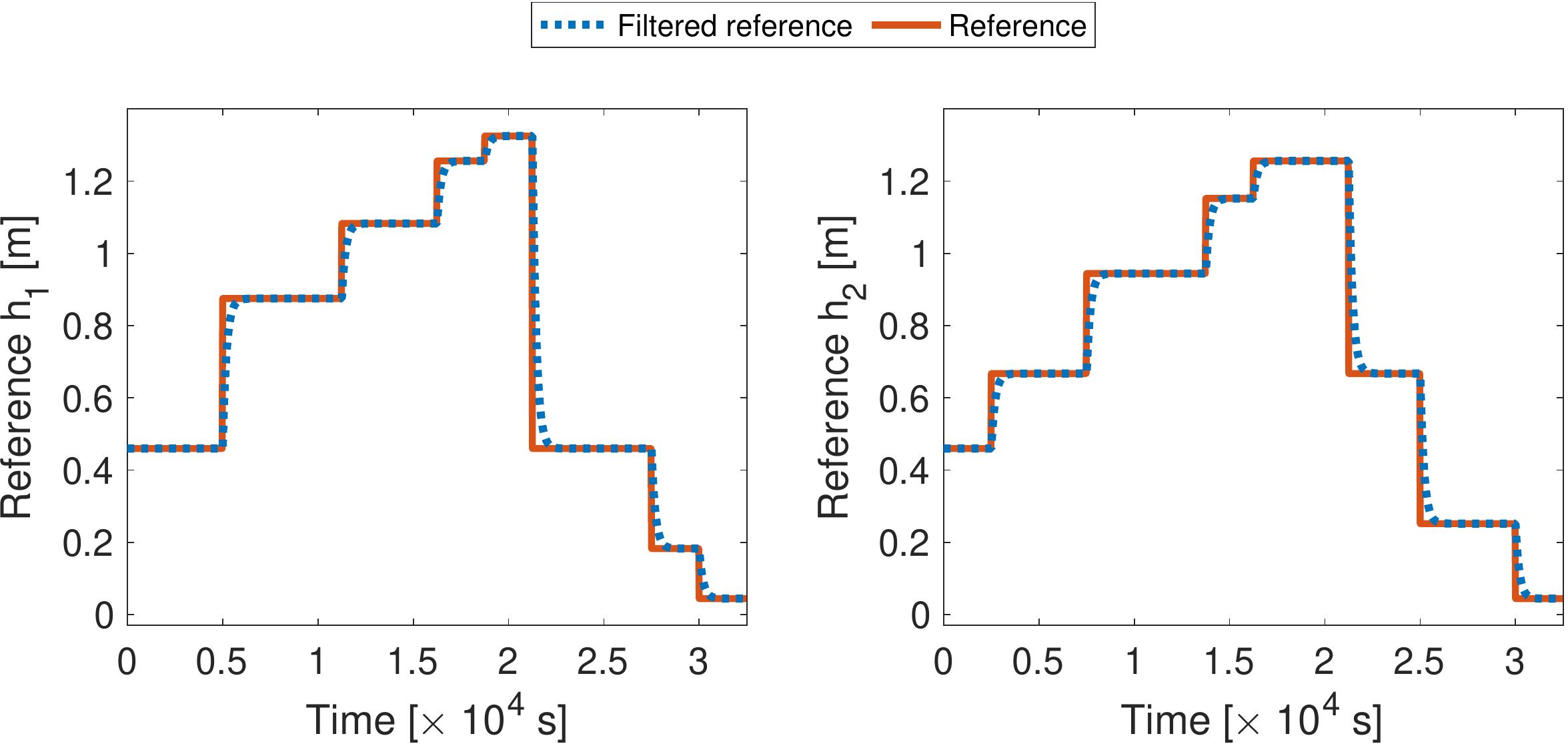}
	\caption{Reference trajectories used for testing the closed-loop performances. Step reference $y^0$ (red line) and filtred reference $\tilde{y}^0$ (blue dotted line). The reference for level $h_1$ is displayed on the left, that for level $h_2$ on the right.}
	\label{fig:closed_loop:references}
\end{figure}

To cope with the measurement noise, the modified architecture shown in Figure \ref{fig:imcscheme_filter} has been considered, where the modeling error feedback $e_m = y_p - y_m$ is customarily filtered by a low-pass filter.
We adopted a low-pass filter $\mathcal{F}$ with the same time constant as the model reference, $\tau_r$. 
This choice is consistent with the IMC literature, see \cite{economou1986internal}.

It is worth noticing that the low-pass filter does not affect the closed-loop stability.
Indeed, the modeling error feedback $e_m$ acts as an additive disturbance on the reference $\tilde{y}^0$ \cite{economou1986internal}, and it converges to the bounded plant-model mismatch (or to zero, in the case of perfect model), see Section \ref{sec:stability}. \smallskip

\begin{figure}[p]
	\centering
	\includegraphics[width=0.8 \linewidth]{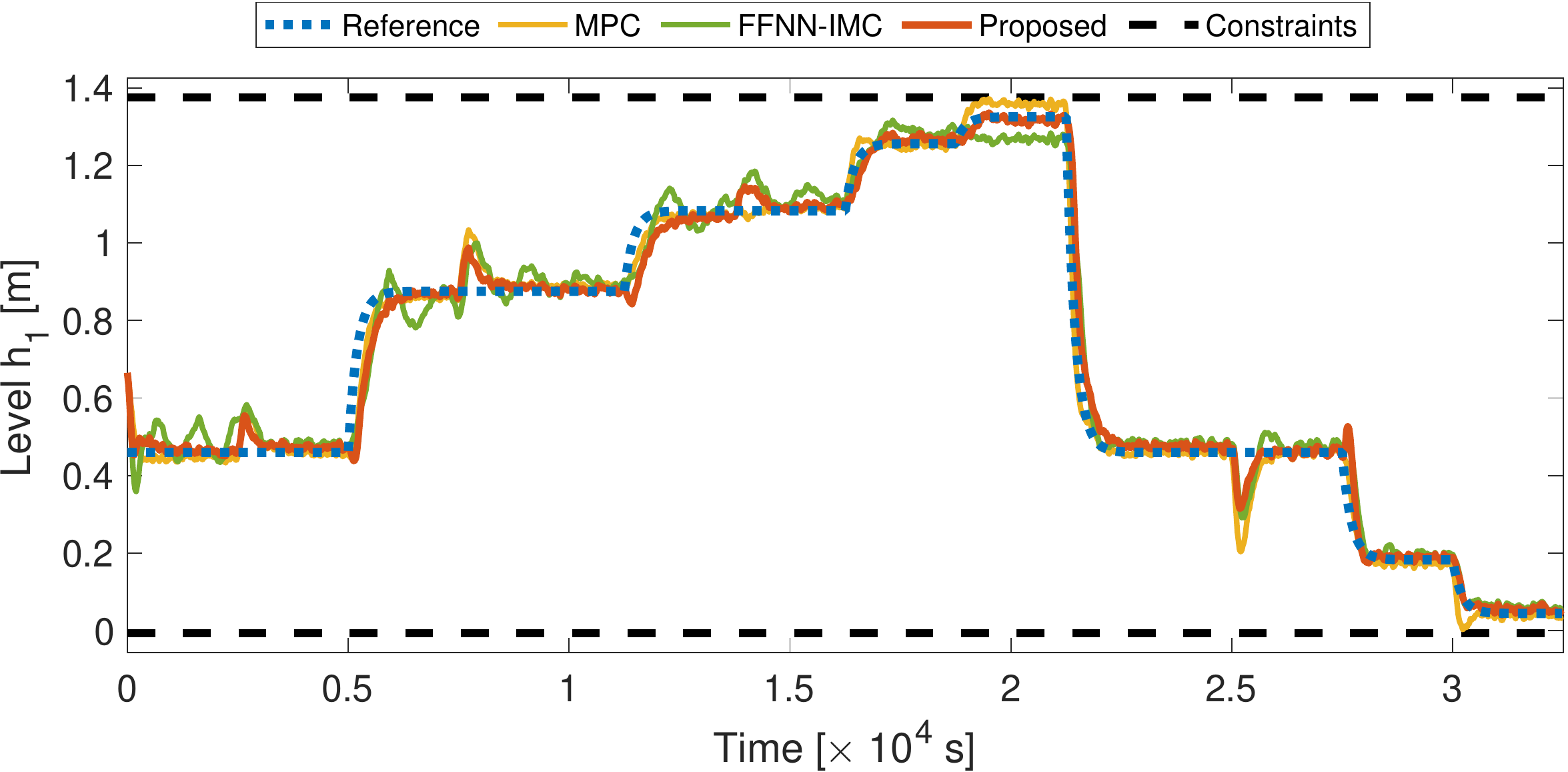}\\
	\vspace{1mm}
	\includegraphics[width=0.8 \linewidth]{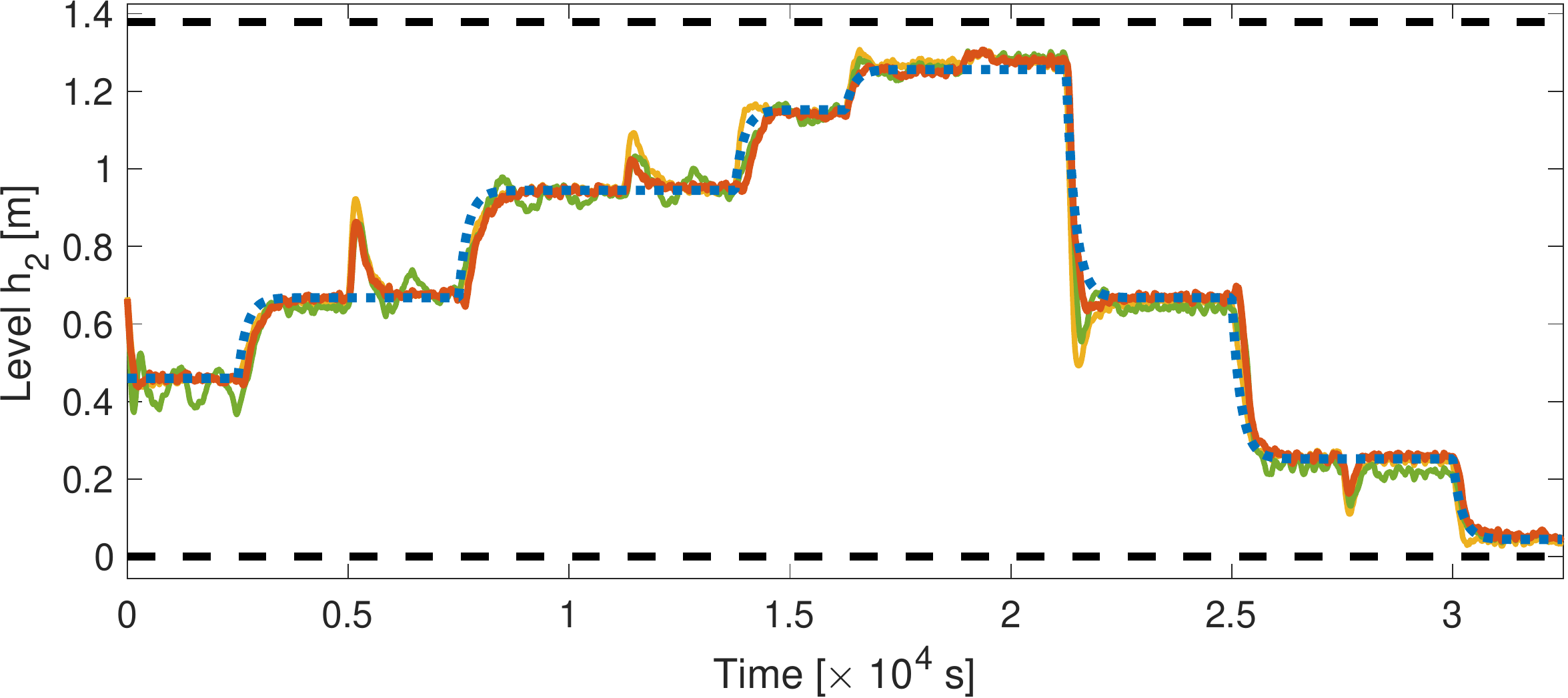}
	\caption{\REV{Closed-loop performances of the three control architectures (MPC, FFNN-IMC, and the proposed IMC approach). The controlled outputs are compared to their reference values and constraints. Level $h_1$ is displayed on top, level $h_2$ on the bottom.}}
	\label{fig:closed_loop:levels}
	\vspace{2mm}
	\centering
	\includegraphics[width=0.8 \linewidth]{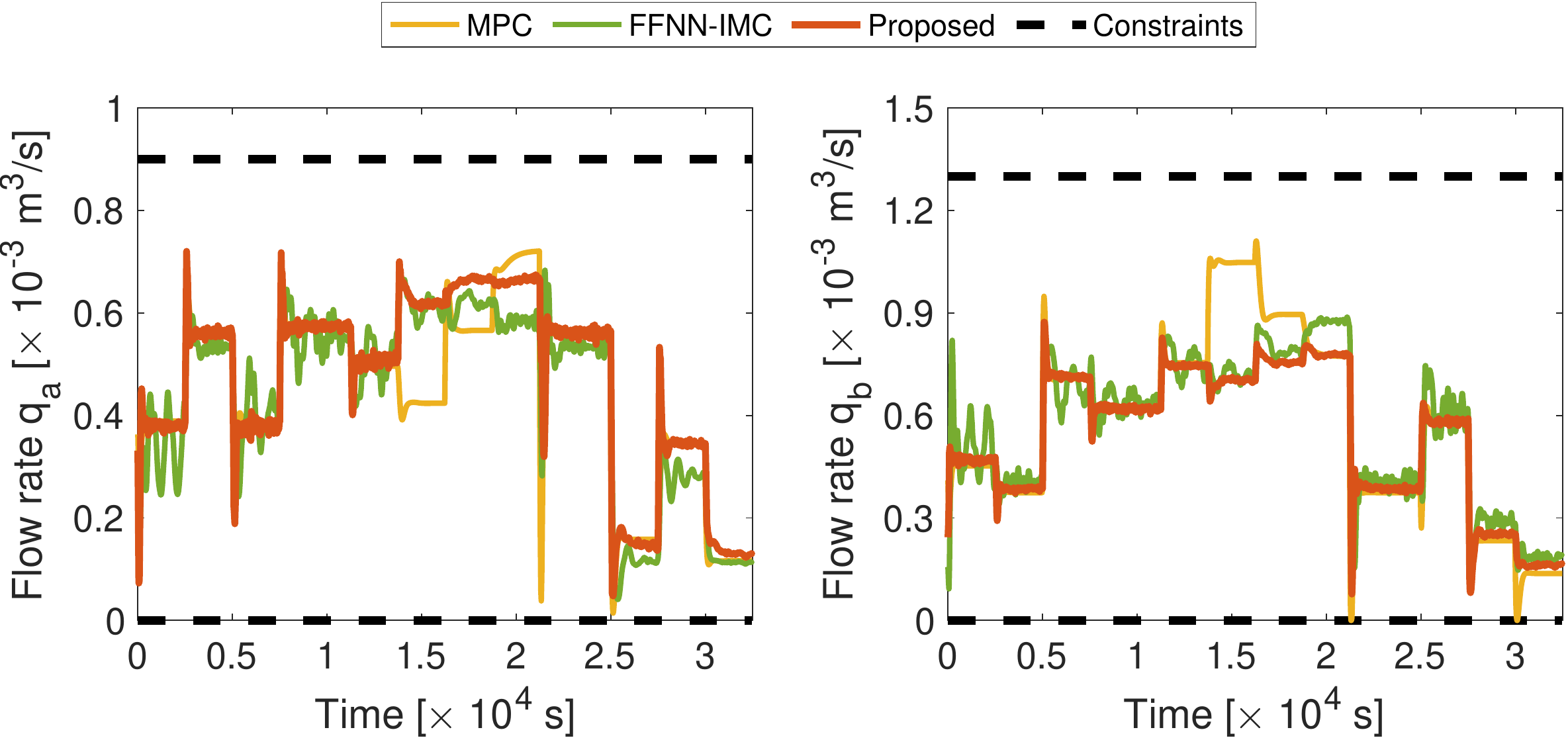}
	\caption{\REV{Comparison of the (denormalized) control action of the three control architectures (MPC, FFNN-IMC, and the proposed IMC approach). Input $q_a$ is displayed on the left, $q_b$ on the right.}}
	\label{fig:closed_loop:u}
\end{figure}

\REV{The closed-loop performances of the proposed approach have been tested and compared to those of two other alternative control architectures:
\begin{enumerate}[a.]
    \item An IMC control architecture realized using FFNNs as system model and controller, on the lines of \cite{rivals2000nonlinear}, adapted to work with MIMO systems; to this end, we assumed a zero control delay and considered FFNNs embedding the $N=6$ past data-points. 
    \item A standard output-tracking nonlinear MPC architecture synthesized using the GRU model $\mathcal{M}$ as predictive model. Since the model's initial state is required to setup the underlying finite-horizon optimization control problem, it has been estimated using a suitably designed state observer for the GRU model, see \cite{bonassi2021nonlinear}.
\end{enumerate}}

For this comparison, the reference trajectories shown in Figure \ref{fig:closed_loop:references} have been adopted, so as to span the set of model's feasible steady-states (depicted in Figure \ref{fig:example:controller_dataset_ss}).

\REV{In Figure \ref{fig:closed_loop:levels} the closed-loop output tracking performances of the three implemented control architectures (i.e. the proposed IMC approach, the MPC, and the FFNN-based IMC) are depicted.
The corresponding tracking error, defined as $e =\tilde{y}^0 - y_p$, is shown in Figure \ref{fig:closed_loop:error}.
}
As expected, cross-couplings are promptly rejected and the controlled outputs are kept close to their reference values, while the controller's architecture allows to satisfy the input saturation constraint\footnote{\REV{In the implemented MPC law, input constraints have been explicitly stated in the optimization problem, while in the FFNN-based IMC, similarly to \eqref{eq:imc:controller:system:output}, a $\tanh$ activation function has been used for the output layer of the controller FFNN.}}, as illustrated in Figure \ref{fig:closed_loop:u}.

\begin{figure}[t]
\centering
\includegraphics[width=0.8\linewidth]{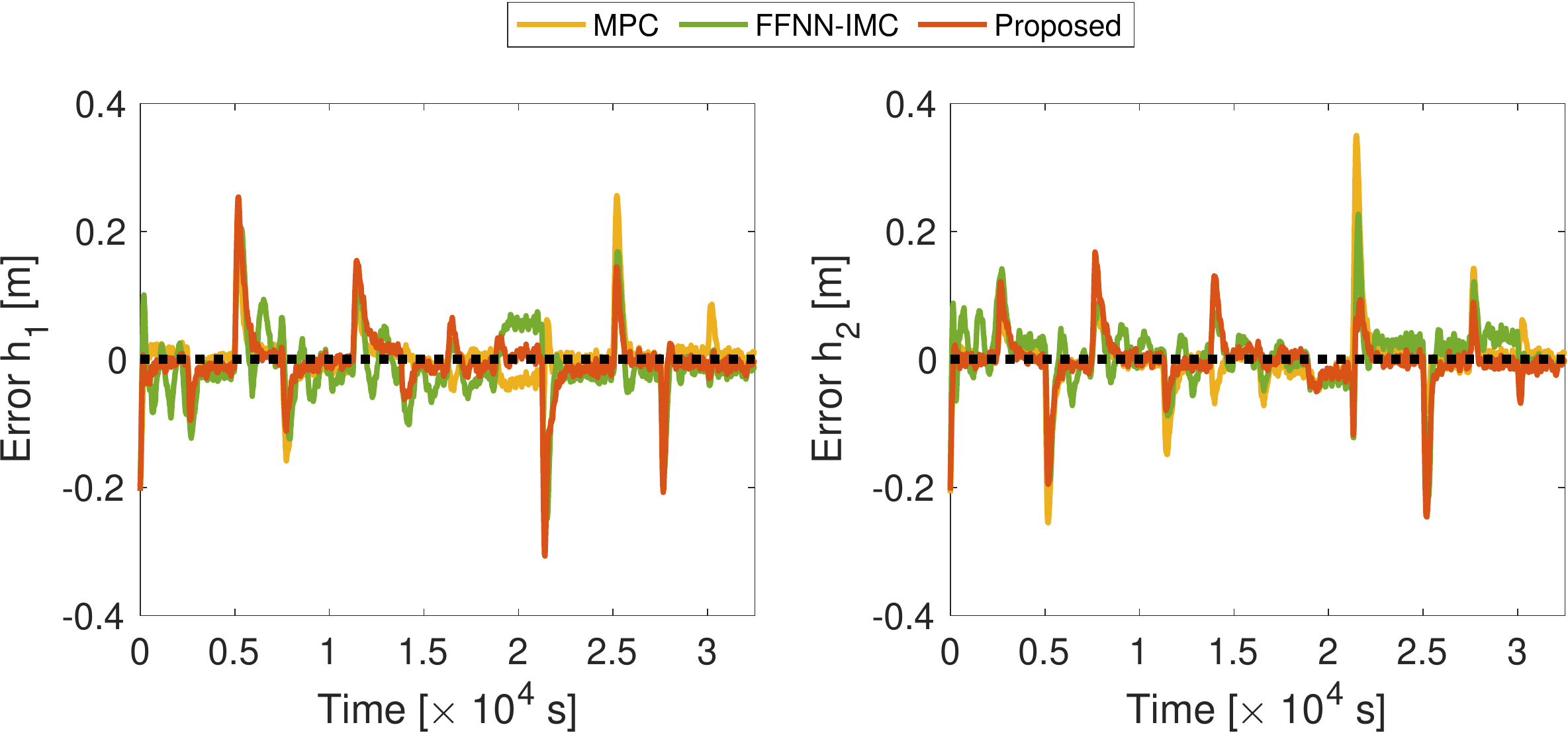}
\caption{\REV{Comparison of the output tracking errors $\tilde{y}^0 - y_p$ for the three control architectures (MPC, FFNN-IMC, and the proposed IMC approach).}}
\label{fig:closed_loop:error}
\end{figure}

\begin{subequations}
\REV{
To evaluate and compare the performances of the three control architectures, the tracking Root-Mean-Square Error (RMSE) is computed as
\begin{equation}
	\epsilon_{tr} = \frac{\| \tilde{\bm{y}}^0_{\scriptscriptstyle T} - \bm{y}_{p, {\scriptscriptstyle T}} \|_{2, 2}}{\sqrt{T}},
\end{equation}
where $T$ indicates the duration of the closed-loop experiment and $\bm{y}_{p, {\scriptscriptstyle T}}$ denotes the closed-loop plant's output. 
The smaller $\epsilon_{tr}$, the better the reference tracking capabilities of the the control scheme.
}

\REV{
Moreover, to evaluate the static performances of the three control architectures, the steady-steady state tracking error has been computed by simulating the closed-loop using the same reference trajectories (depicted in Figure \ref{fig:closed_loop:references}) but removing the gaussian noise affecting the plant's output.
In this way each time the set-point changes, after a sufficiently long transient, the settled closed-loop outputs can be measured and the steady-state error can be computed as
\begin{equation}
	\epsilon_{ss} = \| y_{ss}^0 - y_{p, ss} \|_2,
\end{equation}
where $y_{ss}^0$ denotes the set-point and $y_{p,ss}$ the closed-loop plant's output at steady state. Two static performance indexes can be thus defined as the maximum value of $\epsilon_{ss}$, i.e. $\check{\epsilon}_{ss} \approx \max(\epsilon_ss)$, and its mean value, i.e. $\hat{\epsilon}_{ss} \approx \mathbb{E} [{\epsilon}_{ss}]$, over all the different set-points issued in the noise-free closed-loop simulation.
}
\end{subequations}

\renewcommand{\thefootnote}{\fnsymbol{footnote}}
\begin{table}
	\centering
	\caption{\REV{Performances of the proposed IMC approach}}
	\label{tab:performances}
	\begin{tabular}{l|c|c|c}
		\toprule
		& MPC & FFNN-IMC & Proposed IMC \\
		\midrule 
		 Average computational time [s]\tablefootnote[2]{\REV{Average computational time at each control step. The control architectures have been implemented on a desktop with a 4x4GHz processor and 16Gb of RAM.}} & $3.82$ & $8.4 \cdot 10^{-3}$ & $\bm{1.7 \cdot 10^{-3}}$
 \\
		 Tracking RMSE $\epsilon_{tr}$ [m] & $0.133$ & $0.131$ & $\bm{0.128}$ \\
		 Average steady-state error $\hat{\epsilon}_{ss}$ [m] & $0.82 \cdot 10^{-2}$
& $2.43 \cdot 10^{-2}$  &  $\bm{0.79 \cdot 10^{-2}}$ \\
		 Maximum steady-state error $\check{\epsilon}_{ss}$ [m] & $3.46 \cdot 10^{-2}$ & $5.65 \cdot 10^{-2}$  &  $\bm{2.35 \cdot 10^{-2}}$ \\
		 Closed-loop stability guaranteed & No & No & \textbf{Yes} \\
		 \bottomrule 
	\end{tabular}
\end{table}
\renewcommand{\thefootnote}{\arabic{footnote}}

\REV{
Based on these results, one can conclude that the proposed IMC approach enjoys the following strengths. 
\begin{enumerate}[i.]
    \item \textbf{Performances} -- Owing to the superior modeling capabilities of GRUs,  the proposed IMC approach outperforms the FFNN-based IMC, especially from the steady-state tracking error perspective. Moreover, it also slightly outperforms MPC, probably due to the fact that while MPC assumes the future reference to be constantly equal to the current one, the GRU controller incorporates some knowledge of the future evolution of the reference, having been trained on similar reference trajectories.
    \item \textbf{Computational time} -- As discussed, the computational load of IMC burdens entirely in the synthesis stage. The proposed IMC approach hence requires a limited online computational load, consisting in the propagation of $\mathcal{M}$ and $\mathcal{C}$, beating the slightly higher cost of the FFNN-based IMC, mainly because this latter needs to store a sufficient amount of past data. As expected, the computational burden of MPC is significantly higher, as it requires to solve an online nonlinear optimization problem at each step, that is likely unbearable for low-power embedded boards.
    \item \textbf{Closed-loop stability} -- In light of Proposition \ref{property:stability}, since the model and controller GRUs are trained with $\delta$ISS guarantees, the proposed approach guarantees the input-output stability of the closed-loop. The adopted MPC law does not enjoy guaranteed closed-loop stability, though it is possible to design MPC laws with such guarantee (see e.g. \cite{terzi2021learning, bonassi2020stability}), while no criterion to synthesize a stable FFNN-based IMC has been provided in \cite{rivals2000nonlinear}, especially for MIMO systems.
\end{enumerate}
The advantages of the proposed approach come  at the cost of a more complex training performance. As discussed in \cite{bonassi2020stability}, training provenly-$\delta$ISS networks generally call for a longer training procedure.
}

Lastly, we point out that in applications where offset-free tracking is required, integrators -- equipped with suitable anti-windup mechanisms -- can be placed on the output tracking errors.
This configuration has also been tested in \cite{depari2021design}.
% , where it allowed to achieve offset-free tracking for most references, provided that a careful selection of the integrators' gains is performed to avoid adverse effects on the closed-loop stability.

\section{Conclusion} \label{sec:conclusion}
In this paper, we discussed how Recurrent Neural Networks (RNNs), particularly Gated Recurrent Units (GRUs), can be employed to design an Internal Model Control (IMC) architecture to control a stable nonlinear dynamical system.
The system model retrieval and the model inversion were recast as standard RNNs' training procedures.   
Moreover, recent GRUs' stability results were exploited to ensure the stability of both the model and the controller to guarantee the input-output stability of the closed loop.
The proposed approach was tested on the Quadruple Tank benchmark system \REV{and compared to other alternative control architectures, showing enhanced tracking performances and significantly lower computational times}, while satisfying input saturation constraints.
Future work will be devoted to the use of integrators to attain robust offset-free tracking capabilities.

\section*{Acknowledgments}
The authors are indebted to L. Depari for his contribution in the implementation and testing of the proposed algorithm, see \cite{depari2021design}.
\vspace{2mm}

\begin{minipage}[l]{0.1\textwidth}
	\includegraphics[width=\textwidth]{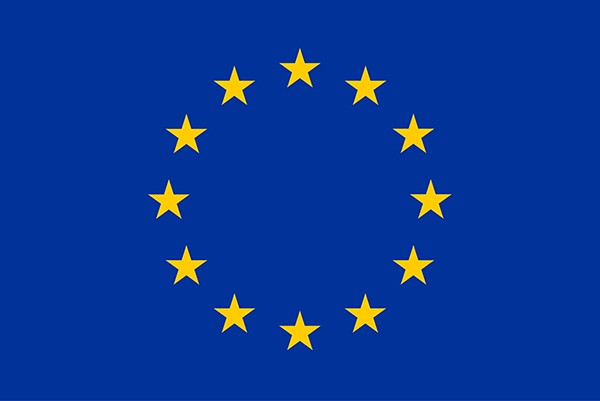}
	\label{fig:euflag}
\end{minipage}
\begin{minipage}[c]{0.05\textwidth}
	\hspace{0.5mm}
\end{minipage}
\begin{minipage}[right]{0.8\textwidth}
	This project has received funding from the European Union's Horizon 2020 research and innovation programme under the Marie Skłodowska-Curie grant agreement No. 953348
\end{minipage}

\bibliographystyle{elsarticle-num}
\bibliography{IMC}

\end{document}